\documentclass[opre,nonblindrev]{informs3aa}
\OneAndAHalfSpacedXI

\usepackage{times}
\let\argmax\relax
\let\argmin\relax
\usepackage{amsmath,amsfonts,bm}

\def\eqref#1{equation~\ref{#1}}
\def\Eqref#1{Equation~\ref{#1}}
\def\plaineqref#1{\ref{#1}}

\def\1{\bm{1}}

\DeclareMathAlphabet{\mathsfit}{\encodingdefault}{\sfdefault}{m}{sl}
\SetMathAlphabet{\mathsfit}{bold}{\encodingdefault}{\sfdefault}{bx}{n}

\newcommand{\R}{\mathbb{R}}

\DeclareMathOperator*{\argmax}{arg\,max}
\DeclareMathOperator*{\argmin}{arg\,min}

\usepackage{amsmath,amssymb}
\usepackage{bm}
\usepackage{booktabs}
\usepackage{float}
\usepackage{enumitem}
\usepackage{graphicx}
\usepackage{url}
\usepackage{tikz}
\usetikzlibrary{arrows.meta, positioning, shapes.multipart}
\graphicspath{{figures/}}

\usepackage{natbib}
\bibpunct[, ]{(}{)}{,}{a}{}{,}
\def\bibfont{\small}

\usepackage[breaklinks=true,bookmarks=false,colorlinks=true,allcolors=blue]{hyperref}

\theoremstyle{TH}
\newtheorem{theorem}{Theorem}
\newtheorem{proposition}[theorem]{Proposition}
\newtheorem{lemma}[theorem]{Lemma}

\theoremstyle{EX}
\newtheorem{definition}[theorem]{Definition}

\let\INFORMSproof\proof
\let\INFORMSendproof\endproof
\renewenvironment{proof}[1][Proof]
  {\INFORMSproof{#1.}}
  {\unskip\nobreak\hfill\ensuremath{\square}\INFORMSendproof}
\EquationsNumberedThrough

\DeclareMathOperator{\MNL}{MNL}
\newcommand{\Nin}{N^-}
\newcommand{\dep}{\mathrm{depth}}
\newcommand{\simplex}[1]{\Delta^{#1}}
\newcommand{\cG}{\mathcal{G}}

\newcommand{\err}{\operatorname{err}}
\newcommand{\ms}[2]{$#1$\,{\scriptsize$(#2)$}}
\newcommand{\msv}[2]{\begin{tabular}{@{}c@{}}$#1$\\[-2pt]{\scriptsize$(#2)$}\end{tabular}}
\newcommand{\ccmodel}{\textsc{cc}}

\RUNAUTHOR{Xiao and Jia}
\RUNTITLE{Consideration Circuits: Depth Separation and Universality Beyond a Single Softmax}
\TITLE{Consideration Circuits: Depth Separation and Universality Beyond a Single Softmax}
\ARTICLEAUTHORS{\AUTHOR{Junjie Xiao$^{1,2}$, Huiwen Jia$^{2}$}
\AFF{$^{1}$Department of Mathematics, Peking University\\
$^{2}$Department of Industrial Engineering and Operations Research, University of California, Berkeley}
}
\ABSTRACT{Most feature-based choice models, classical and deep, score items and
apply a single softmax. We introduce consideration circuits (\ccmodel{}), feature-based models of multi-stage choice defined by directed acyclic graphs of multinomial logit (MNL)
units. Source units assign probabilities to menu items, and internal units combine predecessor distributions using MNL weights computed from their probability-weighted feature summaries.
On a three-item compromise task with fixed non-collinear features,
menu-independent random-utility models (RUM), including a single MNL unit, suffer an error bounded away from zero. For \ccmodel{}, in contrast, we establish a sharp depth--norm
separation: increasing depth from $2$ to $3$ reduces the optimal maximum
taste-vector norm for error $\epsilon$ from
$\Theta(\log(1/\epsilon)/\epsilon)$ to $\Theta(\log(1/\epsilon))$.
The depth-$2$ lower bound holds for arbitrary width and menu-independent
routing biases, while a five-node depth-$3$ circuit with zero routing
biases attains the logarithmic rate.
More generally, we characterize
two geometric conditions that are necessary and sufficient for
approximating arbitrary deterministic choice tables on finite menu
families. Under these conditions, depth $3$ suffices, while depth $4$
achieves optimal logarithmic norm scaling whenever the family contains
a non-singleton menu.
In experiments, standalone tree circuits with
fewer than $600$ parameters attain the lowest mean test negative
log-likelihood (NLL) among the evaluated models on four fixed-pool
benchmarks and the Expedia temporal split.
As output heads, \ccmodel{} generalize the linear MNL readout and lower mean test NLL for every tested encoder on Expedia and Trivago.
}
\begin{document}
\maketitle

\section{Introduction}
\label{sec:intro}

Feature-based choice models retain an MNL-based choice rule while
varying how item utilities are specified: through linear functions
\citep{mcfadden1974conditional}, mixing distributions
\citep{mcfadden2000mixed,kamakura1989probabilistic}, or neural networks
\citep{sifringer2020enhancing,han2022neural,aouad2022representing}.
Context-dependent models further allow an item's utility to depend on
other items in the menu, for example through attention modules that
model interactions among items and update their feature representations
before computing utilities
\citep{wang2023transformer,zhang2026deephalo}. These approaches primarily
increase the flexibility of utility specification, including its
dependence on menu context. A complementary behavioral literature
studies staged choice, including consideration-set formation followed
by choice within the considered set
\citep{tversky1972elimination,payne1976task,sims2003implications,matejka2015rational}.
Classical consideration models often use item-specific parameters on a
fixed, finite item universe
\citep{manzini2014stochastic,masatlioglu2012revealed,brady2016menu},
limiting direct transfer to unseen items. They also typically involve
a single consideration stage. Building on this behavioral literature, we study a feature-based model of multistage choice that supports arbitrary depth and can be trained by stochastic gradient descent. We use the terms \emph{decision depth} and \emph{consideration depth} interchangeably.

We introduce \emph{consideration circuits} (\ccmodel{})
in \S\ref{sec:model}, which connect MNL units in a directed acyclic
graph. Each source unit assigns choice probabilities to the menu
items. Each internal unit combines its predecessors' choice
distributions using MNL weights computed from their
probability-weighted feature summaries. The sink's distribution
gives the final choice probabilities. A depth-2 circuit implements
a differentiable consider-then-choose process, while deeper
circuits introduce additional consideration stages. The units are standard MNL models, but their composition allows more
general choice behavior: even the smallest nontrivial circuit can
violate the regularity axiom of random-utility theory
\citep{block1960random}, reproducing the decoy effect
\citep{huber1982adding}.

A natural question is: what is gained by increasing decision depth, and at what cost? We characterize this tradeoff through upper and lower bounds on the
maximum taste-vector norm required to achieve approximation error
$\epsilon$ (see \S\ref{sec:power}).
On a three-item compromise task with fixed non-collinear
features, independence of irrelevant alternatives gives an error
lower bound of $2-\sqrt2$ for MNL, while regularity gives a lower
bound of $1/2$ for every context-independent random-utility model (RUM). We prove that
depth-$2$ circuits, with one consideration stage before the final
choice, achieve arbitrarily small error with optimal norm
$\Theta(\log(1/\epsilon)/\epsilon)$. Adding one consideration stage
reduces the optimal norm to $\Theta(\log(1/\epsilon))$ at depth $3$.
More generally, we identify two geometric conditions, arising from
linear scoring and softmax translation invariance, that are necessary
for a universality ganranttee. Under these conditions, a fixed depth-$3$ topology,
depending only on the catalog size, can approximate every
deterministic choice table on a finite menu family to arbitrary
accuracy. At depth $4$, the optimal norm is
$\Theta(\log(1/\epsilon))$ whenever the family contains a
non-singleton menu, matching the order required to select one item
from a single such menu.

We first compare tree-structured \ccmodel{} with classical
choice models, many of which assume a fixed item pool
(\S\ref{sec:universality}). With fewer than $600$ parameters, the
circuits achieve the lowest mean test negative log-likelihood (NLL)
among the evaluated models on all four standard benchmarks,
including SFwork and SFshop, previously used to evaluate PCMC and
CDM. We then compare them with deep choice architectures on two
temporally split datasets (\S\ref{sec:vs-deep}). On Trivago,
the strongest deep model improves on the $503$-parameter tree by
$0.023$ NLL with $74$ times as many parameters. On Expedia,
where diagnostics detect stronger temporal shift, every evaluated
deep baseline has higher mean test NLL than MNL.
The trees are the only compared models to outperform both MNL and
LC-MNL on both datasets. The relative robustness of CC under temporal shift suggests that modeling the choice process through staged consideration provides a
useful inductive bias when feature distributions change.

Consideration circuits can also serve as output heads for deep
choice models (\S\ref{sec:head}). Many such models use an encoder
to transform item features, followed by a linear MNL readout.
We retain the encoder architecture and replace its readout with
a tree circuit, training the encoder and circuit jointly.
This lowers mean test NLL for every tested encoder on both
datasets and yields our lowest NLL on Trivago.
In this sense, \ccmodel{} is a \emph{generalized MNL}:
it retains the same interface from item features to choice
probabilities, includes MNL as a special case, and represents
choice behavior that a single MNL cannot.

We summarize our contributions as follows:
\begin{itemize}
  \item \textbf{A feature-based extension of consider-then-choose.}
  We introduce consideration circuits (\ccmodel), which extend
  consider-then-choose to multiple consideration stages.
  Each stage applies an MNL rule to menu items or to the
  probability-weighted feature summaries produced by preceding stages.
  Although the constituent units are standard MNL models, their
  composition accommodates context effects that violate regularity.
  The model supports arbitrary depth and can be estimated by
  stochastic gradient descent.

  \item \textbf{Expressive power of the new model.}
  On a fixed three-item compromise task, we quantify how an additional
  consideration stage reduces the taste-vector norm required for
  approximation.
  Increasing depth from $2$ to $3$ reduces the optimal maximum
  taste-vector norm for error $\epsilon$ from
  $\Theta(\log(1/\epsilon)/\epsilon)$ to
  $\Theta(\log(1/\epsilon))$.
  The depth-$2$ lower bound holds for arbitrary width and
  menu-independent routing biases, whereas a five-node depth-$3$
  circuit with zero routing biases attains the logarithmic rate.
  We characterize exactly when consideration circuits can approximate
  every deterministic choice table on a finite menu family.
  Two geometric conditions, arising from linear scoring and softmax
  translation invariance, are necessary at every depth and sufficient
  at depth $3$.
  Under these conditions, depth $4$ further achieves optimal
  $\Theta(\log(1/\epsilon))$ taste-norm scaling whenever the family
  contains a non-singleton menu.
  These results distinguish restrictions that additional consideration
  stages can overcome from those imposed by the feature representation.

  \item \textbf{Numerical experiments supporting predictive performance.}
  With fewer than $600$ parameters, tree circuits attain the lowest
  mean test NLL among the evaluated models on four fixed-pool
  benchmarks and the Expedia temporal split.
  The same choice specification also applies to learned item
  representations: replacing the linear MNL readout with a circuit,
  and training the encoder and circuit jointly, lowers mean test NLL
  for every tested encoder on both hotel-choice datasets.
\end{itemize}

\subsection{Related Work}
\label{sec:related}

\paragraph{Classical choice models.}
Classical choice models build on the formulations of
\citet{luce1959individual} and \citet{mcfadden1974conditional}.
Extensions account for correlations between alternatives through
nested logit and generalized extreme-value models
\citep{mcfadden1978modeling}, preference heterogeneity through
latent classes and mixing distributions
\citep{kamakura1989probabilistic,mcfadden2000mixed,train2009discrete},
and alternative distributions for utility noise
\citep{alptekinoglu2016exponomial}. Another line of work models
context effects on a fixed set of identified items using Markov
chains \citep{blanchet2016markov,ragain2016pairwise} or context
matrices \citep{tversky1993context,seshadri2019discovering}.
Table~\ref{tab:fixedpool} compares our circuits with representative
models from these families.

\paragraph{Consider-then-choose.}
Behavioral research has long studied staged decision processes
\citep{tversky1972elimination,payne1976task,hauser1990evaluation}.
Econometric models describe consideration through random sets drawn
from a fixed item universe
\citep{manzini2014stochastic,masatlioglu2012revealed,brady2016menu},
identify it from asymmetric demand responses
\citep{abaluck2021consider}, or model attribute-based screening in
conjoint analysis \citep{gilbride2004choice}. Related work studies
assortment optimization under a specified consider-then-choose rule
\citep{aouad2021assortment}. The standard consider-then-choose
formulation has one consideration stage followed by a final choice.
Nested logit also introduces a hierarchy, but uses a predefined
nesting structure and remains a random-utility model
\citep{mcfadden1978modeling}, so our regularity-based lower bound
applies on the compromise task. Our model makes consideration
feature-based and differentiable, and composes it across multiple
stages. \S~\ref{sec:separation} quantifies how an additional
consideration stage reduces the taste-vector norm required for a
given approximation error.

\paragraph{Deep choice models.}
Neural choice models use learned functions to specify item utilities
\citep{sifringer2020enhancing,han2022neural,aouad2022representing}.
Context-dependent models further incorporate information from other
items in the menu through set architectures
\citep{zaheer2017deep,lee2019set} and attention modules
\citep{wang2023transformer,zhang2026deephalo}. These encoders
construct context-dependent item representations before computing
utilities. Consideration circuits also use attention: source units
assign choice probabilities to menu items, and subsequent units
attend to probability-weighted feature summaries from their
predecessors to combine their choice distributions. Attention
therefore implements the successive consideration stages in the
circuit.

\paragraph{Depth separation.}
Prior work studies depth separation through width
\citep{telgarsky2016benefits,eldan2016power} or norm-based
representation costs in unbounded-width ReLU networks
\citep{ongie2020function,parkinson2024depth}. We establish a sharp
norm separation for consideration circuits, whose MNL units combine
choice distributions through probability-weighted feature summaries.
On a fixed three-item compromise task with non-collinear features,
increasing depth from $2$ to $3$ reduces the optimal maximum
taste-vector norm from $\Theta(\log(1/\epsilon)/\epsilon)$ to
$\Theta(\log(1/\epsilon))$. The depth-$2$ lower bound holds for
arbitrary width and routing biases, while a five-node depth-$3$
circuit attains the logarithmic rate (\S~\ref{sec:separation}).

\section{Preliminaries}
\label{sec:prelim}

An \emph{assortment} (i.e., menu) $X=\{x_1,\ldots,x_B\}$ consists of
$B$ items with feature vectors $x_i\in\R^d$. Let
$\mathbf{X}\in\R^{B\times d}$ contain these vectors as rows. A
\emph{choice model} assigns a probability distribution to each menu,
\[
  X \;\longmapsto\; P(\cdot\mid X)\in\simplex{B-1},
\]
where $P(i\mid X)$ is the probability of choosing item $i$ from $X$.
We consider models in which items are identified only by their features.
Reordering the items must therefore reorder their choice probabilities
in the same way: the model is \emph{permutation equivariant}, satisfying
$P(\cdot\mid\Pi\mathbf{X})=\Pi P(\cdot\mid\mathbf{X})$ for every
$B\times B$ permutation matrix $\Pi$.

We use the multinomial logit
(MNL; \citealp{luce1959individual,mcfadden1974conditional}) as the
basic unit of our circuits. Given a taste vector $w\in\R^d$, a finite
set $Z=\{z_1,\ldots,z_m\}\subset\R^d$, and a bias vector $b\in\R^m$,
define
\begin{equation}
  \MNL(w,b;Z)_i \;=\;
  \frac{\exp(w^\top z_i + b_i)}{\sum_{i'=1}^{m}\exp(w^\top z_{i'} + b_{i'})},
  \qquad i=1,\ldots,m.
  \label{eq:mnl}
\end{equation}
This gives a choice distribution $\MNL(w,b;Z)\in\simplex{m-1}$ over
$Z$. Our circuits apply this operator to both menu items and intermediate
feature summaries. When $b=c\mathbf{1}$, the common bias cancels in the
softmax, giving the same distribution as zero bias. We denote this case
by $\MNL(w,\mathrm{u};Z)$.

\section{The Consideration Circuit}
\label{sec:model}

\begin{definition}[Consideration graphs]
\label{def:consideration-graphs}
Let $\cG_1$ denote the class of finite, weakly connected directed
acyclic graphs $G=(V,E)$ with a unique vertex $v^\star$ of out-degree
$0$, called the \emph{sink}. We write
$\Nin(v)=\{u:(u,v)\in E\}$ for the in-neighborhood of $v$. Vertices
with $\Nin(v)=\varnothing$ are \emph{sources}. The \emph{depth} of
$G$, denoted by $\dep(G)$, is the number of vertices on its longest
directed path.
\end{definition}

A \emph{consideration circuit} on $G\in\cG_1$ has a trainable taste
vector $w_v\in\R^d$ at each vertex $v\in V$ and a trainable bias
$b_{uv}\in\R$ on each edge $(u,v)\in E$. The total number of
trainable parameters is $|V|d+|E|$. We write
$b_{\cdot v}=(b_{uv})_{u\in\Nin(v)}$ for the biases on the incoming
edges of $v$. Given an assortment $X$, the circuit is evaluated in
topological order:
\begin{align}
  \text{source } v:\quad
  & p(v) = \MNL(w_v,\mathrm{u};X), \qquad x(v) = p(v)^\top\mathbf{X};
  \label{eq:fwd-source}\\
  \text{interior } v:\quad
  & q(v) = \MNL\bigl(w_v,b_{\cdot v};\{x(u):u\in\Nin(v)\}\bigr),
  \label{eq:fwd-interior}\\
  & p(v) = \textstyle\sum_{u\in\Nin(v)} q(v)_u\, p(u),
  \qquad x(v) = \textstyle\sum_{u\in\Nin(v)} q(v)_u\, x(u).
  \notag
\end{align}
The circuit outputs $p(v^\star)$. A single-vertex circuit reduces
to the bias-free MNL $\MNL(w,\mathrm{u};X)$.

\begin{lemma}[Consistency]
\label{lem:consistency}
For every $v\in V$, $x(v)=p(v)^\top\mathbf{X}$. Every intermediate
$p(v)$ is a valid choice distribution over the original $B$
alternatives, and the circuit is permutation equivariant.
\end{lemma}

\paragraph{A differentiable consider-then-choose process.}
Consider the simplest nontrivial structure: a depth-$2$ star with
$Q$ source units feeding a single sink. In the hard-max limit,
where each source MNL selects an argmax, $p(v)$ is one-hot and
$x(v)$ is the feature vector of the selected item. The sources
thus form a \emph{consideration set} from the menu, and the sink
chooses from this set. With softmax, the circuit implements a
differentiable version of this process. Each source acts as a
\emph{consideration proposal}, supplying the sink with a generalized
item represented by the probability-weighted feature summary
$x(v)=p(v)^\top\mathbf{X}$. The sink then chooses within the
resulting generalized consideration set. This interpretation also
explains the bias asymmetry. Source units have no position-specific
biases because menu positions have no intrinsic meaning. The sink
scores source units that may represent different preferences, such
as higher quality or lower price. Its edge biases $b_{\cdot v}$
therefore provide menu-independent priors over these units. A general
circuit repeats this process: each interior unit chooses among the
generalized items supplied by its predecessors and passes its own
summary $x(v)$ to its successors.

\paragraph{Comparison with mixtures of experts.}
On a depth-$2$ star, the circuit has the same mixture structure as a
mixture of experts (MoE): source units produce choice distributions,
and the sink combines them using MNL weights. The key difference is
how the menu determines these weights. In our circuit, the sink
scores source $v$ by $w^{\top}x(v)+b_v$, pairing a menu-dependent
\emph{key} $x(v)$, computed from the source's choice distribution,
with a learned \emph{query} $w$. In the corresponding MoE, expert
$v$ receives score $a_v^{\top}g(X)+b_v$: its key $a_v$ is learned
and menu-independent, while its query $g(X)$ encodes the menu.
Figure~\ref{fig:hme} compares these two mechanisms. The circuit
also requires fewer per-source parameters: a source learns only
its taste vector and uses its summary as its key, whereas an expert
in this MoE learns both a taste vector and a separate key. The
comparison extends to hierarchical mixtures of experts
(HMEs; \citealp{jacobs1991adaptive,jordan1994hierarchical}). In the
circuit, each interior unit chooses among its predecessors' summaries
and passes its weighted summary toward the sink. In the corresponding
HME, each gate uses a query encoding the menu to distribute mixture
weights from the root toward the experts.

\begin{figure}[htbp]
\centering
\includegraphics[width=0.9\linewidth]{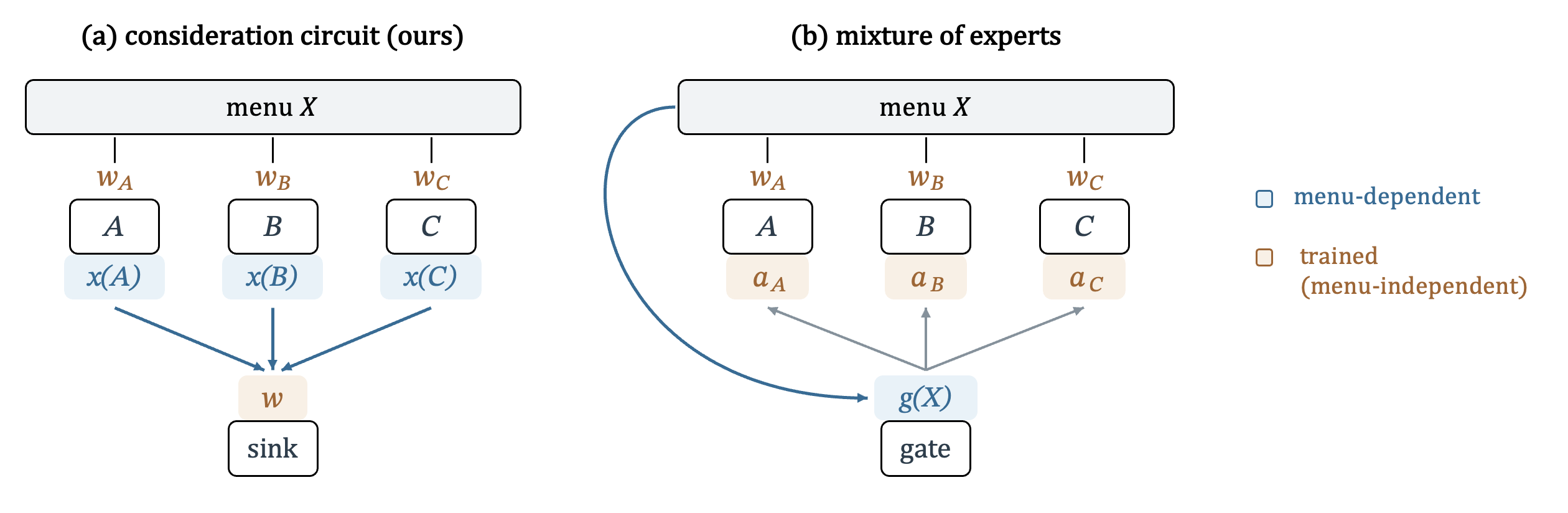}
\caption{Depth-$2$ consideration circuit versus mixture of experts.
Colors indicate menu-dependent (blue) and trained, menu-independent
(orange) quantities.
\textbf{(a)}~The circuit sink scores source $v$ by
$w^{\top}x(v)+b_v$, with menu-dependent key $x(v)$ and trained query
$w$.
\textbf{(b)}~The MoE gate scores expert $v$ by
$a_v^{\top}g(X)+b_v$, with trained key $a_v$ and menu-dependent query
$g(X)$.}
\label{fig:hme}
\end{figure}

\paragraph{Breaking regularity: the decoy effect.}
Even a circuit with just two sources $A$, $B$ and a sink, with both
edge biases fixed to zero, can produce behavior excluded by
context-independent random-utility models. Such models satisfy
regularity: adding an item to a menu cannot increase an existing
item's choice probability \citep{block1960random}. The empirically
observed decoy (asymmetric-dominance) effect violates this property
\citep{huber1982adding}. Figure~\ref{fig:decoy} shows how the circuit
reproduces this effect. On menu $\{1,2\}$, source $A$ favors item
$1$, source $B$ favors item $2$, and the sink assigns the sources
nearly equal weights, giving each item a $50\%$ share. Adding decoy
$3$, which is dominated by item $2$, does not change either source's
most probable item. It does, however, change their feature summaries.
Source $A$ assigns the decoy probability approximately $0.12$,
shifting its summary $x(A)=p(A)^\top\mathbf{X}$ and lowering the
sink's score for $A$ (Figure~\ref{fig:decoy}c). The sink chooses among
these summaries and shifts its attention from $A$ toward $B$,
raising item $2$'s share from $50\%$ to $60\%$. All item features
remain fixed: adding the decoy changes the source summaries and
hence the sink's weights. All probabilities in Figure~\ref{fig:decoy}
are computed from a fitted circuit that matches the target shares
to numerical precision (Appendix~\ref{app:decoy-fit}).

\begin{figure}[htbp]
\centering
\includegraphics[width=0.9\linewidth]{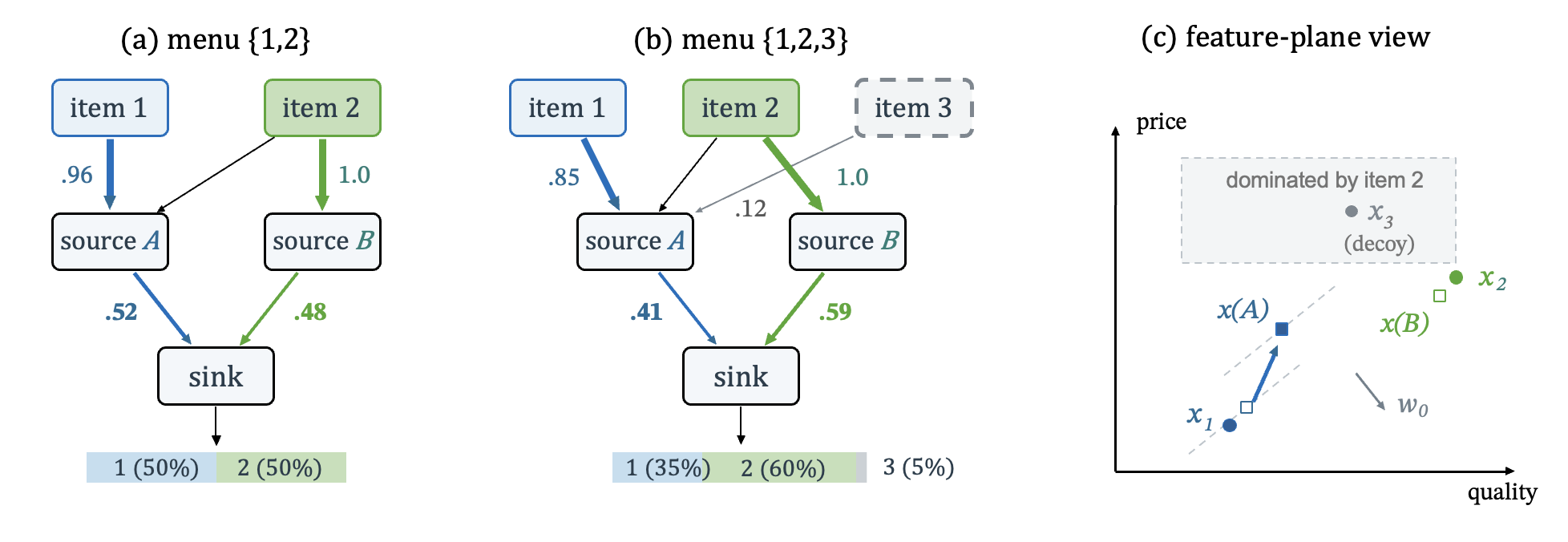}
\caption{The smallest consideration circuit fitted to a decoy
instance. All probabilities are computed from the fitted parameters
(Appendix~\ref{app:decoy-fit}); edge widths are proportional to
probabilities, edges below $0.01$ omitted. \textbf{(a)}~On the menu
$\{1,2\}$ both items have a $50\%$ share. \textbf{(b)}~The decoy $3$ shifts the
sink's attention toward source $B$: item $2$'s share rises to
$60\%$, a regularity violation. \textbf{(c)}~The mechanism
(schematic): the decoy draws part of $A$'s attention, moving $x(A)$
across the sink's iso-utility lines (dashed), while $x(B)$ remains close to $x_2$.}
\label{fig:decoy}
\end{figure}

\section{Structural Results}\label{sec:power}

We quantify the expressive power of consideration circuits and the
taste-vector norm required for approximation for different depth. On a three-item
compromise task, we establish a sharp separation between depths
$2$ and $3$. We then identify necessary and sufficient geometric
conditions for universality over deterministic choice tables, show
that depth $3$ suffices, and establish optimal logarithmic norm
scaling at depth $4$.

\subsection{A Depth Separation}
\label{sec:separation}

A single consideration stage already allows depth-$2$ circuits to
violate regularity. We now ask how an additional stage changes the
norm required to approximate a choice rule. Consider three items
with non-collinear features $x_1,x_2,x_3\in\R^d$. The compromise task
\citep{simonson1989choice} requires item $3$ to lose both pairwise
comparisons but be chosen from the full menu:
\begin{equation}
  \{1,3\}\mapsto 1,\qquad
  \{2,3\}\mapsto 2,\qquad
  \{1,2,3\}\mapsto 3.
  \label{eq:compromise-task}
\end{equation}
A model solves this task with error at most $\epsilon$ if it assigns
probability at least $1-\epsilon$ to the target item on each menu.
Figure~\ref{fig:five-node-hard} shows a fixed five-node depth-$3$
circuit that realizes the task exactly in the hard-max limit. Both
depths $2$ and $3$ permit arbitrarily small error; we compare the
taste-vector norm required. We measure this by $\max_v\|w_v\|$,
allowing arbitrary menu-independent edge biases without charging
them to this budget. Let $N_k^*(\epsilon)$ be the infimum of this norm
over all finite circuits of depth at most $k$ and error at most
$\epsilon$, with arbitrary width and fan-in:
\begin{equation}
  N^*_k(\epsilon)
  \;=\;
  \inf_{\substack{\text{circuits of depth}\,\le\,k\\
                  \text{error at most }\epsilon}}
  \;\max_{v}\,\|w_v\|.
  \label{eq:budget-def}
\end{equation}

\begin{theorem}[Depth separation]
\label{thm:separation}
For every non-collinear $x_1,x_2,x_3$,
\[
  N^*_2(\epsilon)
  =\Theta_X\!\left(\frac{\log(1/\epsilon)}{\epsilon}\right),
  \qquad
  N^*_3(\epsilon)
  =\Theta_X\!\left(\log\frac1\epsilon\right),
\]
where $\Theta_X$ hides constants depending only on the geometry of
the features ($\Omega_X$, $O_X$: the one-sided versions). The depth-$2$ lower bound holds for stars of arbitrary
width and arbitrary edge biases; the depth-$3$ upper bound is
attained by the five-node graph of
Figure~\ref{fig:five-node-hard} with all edge biases zero.
\end{theorem}

On this task, independence of irrelevant alternatives forces plain
MNL to have error at least $2-\sqrt{2}$, while regularity gives a
lower bound of $1/2$ for every context-independent random-utility
model. Depth $2$ permits arbitrarily small error, but its optimal norm
includes an additional $1/\epsilon$ factor. Neither arbitrary width
nor unrestricted edge biases can remove this factor; a fixed
five-node depth-$3$ circuit with zero edge biases does. The resulting
logarithmic norm scaling is optimal among circuits of any depth.
Indeed, the output is a convex combination of source distributions,
so a target probability of at least $1-\epsilon$ requires some source
to assign at least that probability. This already requires norm
$\Omega_X(\log(1/\epsilon))$. Proofs are given in
Appendix~\ref{app:power}.

\begin{figure}[htbp]
\centering
\includegraphics[width=0.9\linewidth]{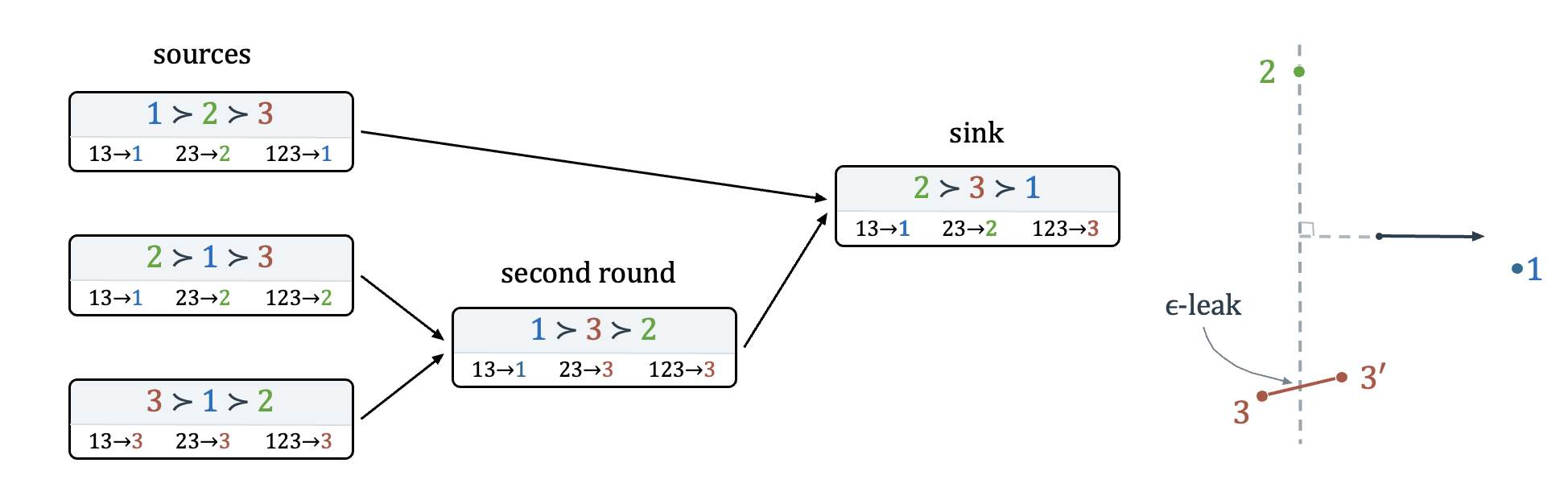}
\caption{\textbf{Left:} a hard structure realizing the task
(\plaineqref{eq:compromise-task}): the five-node depth-$3$ graph in
the hard-max limit, all edge biases zero. In this limit every summary
is itself an item, and every node reduces to a strict preference
ranking: each block shows a node's ranking (top) and the item it
reports on each menu (bottom); a node reports its highest-ranked
among the items nominated by its in-neighbors. The sink's row is
exactly the target. \textbf{Right:} a matching depth-$2$ construction for the same task
(schematic). This construction exploits an $\epsilon$-attention
leak: the source summary near $x_3$ changes by $O_X(\epsilon)$
across menus. The sink's indifference line through $x_2$ (dashed)
separates $x_3$ and $x_{3'}$, allowing the sink to favor the summary
near $x_2$ on $\{2,3\}$ and the summary near $x_3$ on
$\{1,2,3\}$. Amplifying this displacement requires an additional
$1/\epsilon$ factor in norm.}
\label{fig:five-node-hard}
\end{figure}

\paragraph{Depth on real data.}
We evaluate the effect of depth on Expedia and Trivago hotel-search
sessions at a fixed budget of $63$ nodes. For depths $2$ through $6$,
we compare trees with fully connected layered DAGs. Tree
connectivity can be obtained from the corresponding dense graph by
masking selected edges with routing biases of $-\infty$. Structures,
datasets, and the training protocol are given in
Appendix~\ref{app:protocol}.

Table~\ref{tab:depth-ladder} shows substantially larger gains from
increasing depth than from adding edges at fixed depth. Increasing
depth from $2$ to $3$ reduces test NLL by $0.0109$--$0.0160$, whereas
tree and dense circuits at the same depth differ by at most $0.0023$.
Further depth brings little improvement on Expedia, while mean test
NLL on Trivago decreases at every tested depth. We leave a theoretical
analysis of connectivity to future work. We use complete binary trees
in subsequent experiments: the depth-$6$ tree uses $503$ parameters
versus $1{,}123$ for its dense counterpart and achieves the lowest
mean test NLL on Trivago.

\begin{table}[t]
\centering\small
\begin{tabular}{c cc cc}
\toprule
& \multicolumn{2}{c}{Expedia} & \multicolumn{2}{c}{Trivago}\\
\cmidrule(lr){2-3}\cmidrule(lr){4-5}
depth & tree & dense & tree & dense\\
\midrule
2 & \ms{2.6101}{0.0008} & \ms{2.6100}{0.0011}
  & \ms{2.4540}{0.0004} & \ms{2.4539}{0.0005}\\
3 & \ms{2.5980}{0.0021} & \ms{2.5991}{0.0015}
  & \ms{2.4389}{0.0015} & \ms{2.4379}{0.0021}\\
4 & \ms{2.5973}{0.0022} & \ms{2.5983}{0.0017}
  & \ms{2.4357}{0.0020} & \ms{2.4334}{0.0019}\\
5 & \ms{2.5975}{0.0027} & \ms{2.5990}{0.0012}
  & \ms{2.4334}{0.0017} & \ms{2.4329}{0.0015}\\
6 & \ms{2.5988}{0.0014} & \ms{2.5991}{0.0010}
  & \ms{2.4308}{0.0019} & \ms{2.4323}{0.0016}\\
\bottomrule
\end{tabular}
\caption{Test NLL, mean {\scriptsize(sd)} over $10$ restarts, of
$63$-node circuits of depth $2$--$6$: trees vs.\ fully-connected
layered DAGs (dense); at depth $2$ the two coincide. Trees share one
parameter count ($503$); dense graphs pay one bias per edge (up to
$1153$). The depth-$6$ tree is a complete binary tree.}
\label{tab:depth-ladder}
\end{table}

\subsection{Universality}
\label{sec:universality}

We now characterize when consideration circuits can approximate
arbitrary deterministic choice tables. Two properties of linear MNL
units constrain universality at every depth. First, if an item's
feature lies in the convex hull of the other offered features, no
source MNL can select it with probability approaching one. Circuit
composition preserves this restriction. Second, softmax is invariant
to a common score shift, so translated menus receive identical
probabilities under the corresponding item translation. Their target
choices therefore cannot be assigned independently. An instance, i.e., a
catalog $x_1,\ldots,x_N\in\R^d$ and an observed menu family
$\mathcal{F}$, is \emph{admissible} if \textnormal{(E)} every
offered feature lies outside the convex hull of the other offered
features, and \textnormal{(T)} no two distinct non-singleton menus
are exact translates. These conditions are necessary at every depth.
The next theorem shows that they already suffice at depth $3$, with
a topology depending only on $N$.

\begin{theorem}[Universality at depth $3$]
\label{thm:universality3}
For every $N$ there is one topology $G^{(3)}_N$ of depth at most $3$
with at most $2^{N+1}$ nodes, depending only on $N$, such that on
every admissible instance, every deterministic choice table (an
arbitrary assignment of one target item to each observed menu) is
achieved to every accuracy $0<\epsilon<1$. For a fixed instance, a sufficient norm bound is
$\max_v\|w_v\|\le C_X\,\epsilon^{-\kappa}\log(1/\epsilon)$, with
$C_X,\kappa<\infty$ depending only on the geometry.
\end{theorem}

Depth $3$ achieves universality, but logarithmic norm does not suffice
in general. There is an admissible planar instance ($60$ points,
$48$ menus) and a deterministic choice table for which \emph{every}
fixed finite depth-$3$ topology requires
$\omega(\log(1/\epsilon))$ taste-vector norm, even with arbitrary
menu-independent edge biases. This limitation therefore applies
beyond our construction. For affinely independent catalogs, however,
depth $3$ already achieves logarithmic norm.

\begin{proposition}[Affinely independent catalogs]
\label{prop:onehot}
For every $N$ there is one topology of depth at most $3$ with
$N2^{N-1}-N+2$ nodes, depending only on $N$, such that for every
affinely independent
catalog (one-hot items $x_i=e_i$ being the canonical case; such
catalogs are automatically admissible), every deterministic choice
table on every menu family is achieved to every accuracy $0<\epsilon<1$ with
$\max_v\|w_v\|\le C_X\log(1/\epsilon)$.
\end{proposition}

One additional layer achieves logarithmic norm on every admissible instance:

\begin{theorem}[Universality at depth $4$ with optimal norm]
\label{thm:universality4}
Fix $N$ and an abstract menu family
$\mathcal{F}\subseteq 2^{[N]}$ with $|\mathcal{F}|=m$. There is one
topology $G^{(4)}_{\mathcal{F}}$ of depth at most $4$ with
$m^2+m+3$ nodes, depending only on $\mathcal{F}$, independently of the features and
target table, such that for every
catalog $x_1,\ldots,x_N$ making the instance admissible, every
deterministic choice table on $\mathcal{F}$ is achieved to every accuracy
$0<\epsilon<1$ with $\max_v\|w_v\|\le C_X\log(1/\epsilon)$; for any
family containing a non-singleton menu this is the optimal order,
by the single-menu floor of \S\ref{sec:separation}.
\end{theorem}

Together, these results give necessary and sufficient conditions for
universality over deterministic choice tables. Depth $3$ suffices for
approximation, while depth $4$ achieves the optimal logarithmic norm
dependence on $1/\epsilon$, matching the order required to select a
single item from a single non-singleton menu. Proofs are deferred to
Appendix~\ref{app:power}.

\paragraph{Comparison with established choice models.}
Built entirely from MNL units, consideration circuits provide a
direct comparison with established choice models. Many of these
models assign parameters to item identities and therefore assume a
fixed item pool. We evaluate on four standard fixed-pool benchmarks:
SFwork and SFshop, which use item identities, and Swissmetro and
ModeCanada, which include numeric attributes. We use two complete
binary trees, $\mathrm{Tree}_{15}$ and $\mathrm{Tree}_{63}$.
Item identities allow per-item source biases in our circuits
(\S\ref{sec:model}); on the featured datasets, every model includes
item intercepts.

Table~\ref{tab:fixedpool} reports test NLL over $10$ restarts. The
circuits achieve the lowest mean NLL on all four datasets:
$\mathrm{Tree}_{15}$ on SFwork and SFshop, and
$\mathrm{Tree}_{63}$ on Swissmetro and ModeCanada. The largest gains
occur on the featured datasets, where $\mathrm{Tree}_{63}$ improves
over the strongest baseline by $0.0251$ and $0.0253$ NLL,
respectively. The circuits use $98$--$566$ parameters, compared with
$6$--$288$ for the baselines (Appendix~\ref{app:fixedpool-params}).

\begin{table}[htbp]
\centering\small
\setlength{\tabcolsep}{6pt}
\begin{tabular}{l cc cc}
\toprule
& \multicolumn{2}{c}{featureless} & \multicolumn{2}{c}{featured}\\
\cmidrule(lr){2-3}\cmidrule(lr){4-5}
model & SFwork & SFshop & Swissmetro & ModeCanada\\
\midrule
MNL            & \ms{0.8299}{.0002} & \ms{1.5887}{.0007} & \ms{0.8085}{.0038} & \ms{0.6304}{.0041}\\
NestedLogit    & \ms{0.8321}{.0011} & \ms{1.5898}{.0010} & \ms{0.8079}{.0038} & \ms{0.6268}{.0010}\\
LC-MNL         & \ms{0.8262}{.0005} & \ms{1.5669}{.0010} & \ms{0.7518}{.0083} & \ms{0.5929}{.0034}\\
MMNL           & \ms{0.8291}{.0010} & \ms{1.5866}{.0009} & \ms{0.7696}{.0004} & \ms{0.6212}{.0011}\\
Exponomial     & \ms{0.8315}{.0007} & \ms{1.5841}{.0014} & \ms{0.7850}{.0006} & \ms{0.6295}{.0014}\\
ConsiderMNL    & \ms{0.8268}{.0018} & \ms{1.5873}{.0021} & \ms{0.7786}{.0235} & \ms{0.5946}{.0099}\\
CDM-id         & \ms{0.8239}{.0016} & \ms{1.5708}{.0028} & \ms{0.7875}{.0033} & \ms{0.6221}{.0017}\\
LowRankCDM     & \ms{0.8211}{.0015} & \ms{1.5648}{.0017} & \ms{0.7573}{.0023} & \ms{0.5843}{.0027}\\
PCMC           & \ms{0.8243}{.0007} & \ms{1.5754}{.0006} & \ms{0.8615}{.0016} & \ms{0.9122}{.0006}\\
MCCM           & \ms{0.8233}{.0004} & \ms{1.5688}{.0003} & \ms{0.8564}{.0019} & \ms{0.9177}{.0010}\\
\midrule
$\mathrm{Tree}_{15}$ & \ms{\bm{0.8209}}{.0014} & \ms{\bm{1.5644}}{.0025} & \ms{0.7413}{.0087} & \ms{0.5719}{.0093}\\
$\mathrm{Tree}_{63}$ & \ms{0.8215}{.0012} & \ms{1.5661}{.0023} & \ms{\bm{0.7267}}{.0066} & \ms{\bm{0.5590}}{.0063}\\
\bottomrule
\end{tabular}
\caption{Test NLL, mean {\scriptsize(sd)} over $10$ restarts, on four
fixed-pool benchmarks; best per column in bold.
$\mathrm{Tree}_{15}$/$\mathrm{Tree}_{63}$ are complete binary trees
with $15$/$63$ nodes; on the featured pair every model carries item
intercepts. LC-MNL and LowRankCDM show the best size per dataset
($K\in\{2,3,8,16,32\}$, rank $\in\{1,2,3,4\}$); full grids in
Appendix~\ref{app:protocol}.}
\label{tab:fixedpool}
\end{table}

\section{Numerical Experiments}\label{sec:six}

We evaluate consideration circuits in two roles: as compact choice
models compared with deep architectures (\S\ref{sec:vs-deep}), and
as output heads that integrate with learned feature representations
(\S\ref{sec:head}).

\subsection{Comparison with Deep Choice Architectures}
\label{sec:vs-deep}

We compare $\mathrm{Tree}_{15}$ and $\mathrm{Tree}_{63}$ with five
deep choice architectures on Expedia and Trivago, using the temporal
splits of \S\ref{sec:separation} and the protocols in
Appendix~\ref{app:protocol}. MNL and its latent-class extension
LC-MNL serve as classical baselines. The circuit architectures
remain unchanged. Table~\ref{tab:deep} reports test NLL and
parameter counts.

\begin{table}[H]
\centering\small
\setlength{\tabcolsep}{3.5pt}
\begin{tabular}{l cc ccccc cc}
\toprule
& \multicolumn{2}{c}{classical} & \multicolumn{5}{c}{deep}
& \multicolumn{2}{c}{ours}\\
\cmidrule(lr){2-3}\cmidrule(lr){4-8}\cmidrule(lr){9-10}
& MNL & LC-MNL & DeepMNL & RUMnet & DeepSets & DeepHalo & TCNet
& $\mathrm{Tree}_{15}$ & $\mathrm{Tree}_{63}$\\
\midrule
params & 8 & 64 & 4.7k & 4.5k & 3.7k & 201k & 20k/37k & 119 & 503\\
\midrule
Expedia & \msv{2.6574}{.0008} & \msv{2.6153}{.0006}
        & \msv{3.0804}{.5486} & \msv{2.6666}{.0412}
        & \msv{2.9113}{.2282} & \msv{2.9592}{.2421}
        & \msv{2.6911}{.0157}
        & \msv{2.6009}{.0026} & \msv{\bm{2.5983}}{.0018}\\
\addlinespace
Trivago & \msv{2.6400}{.0006} & \msv{2.4631}{.0006}
        & \msv{2.4746}{.0006} & \msv{2.4441}{.0013}
        & \msv{2.4179}{.0032} & \msv{2.4140}{.0027}
        & \msv{\bm{2.4101}}{.0033}
        & \msv{2.4417}{.0015} & \msv{2.4331}{.0023}\\
\bottomrule
\end{tabular}
\caption{Test NLL, mean {\scriptsize(sd)}, on the two temporally
split datasets; best per dataset in bold. MNL, LC-MNL, DeepSets,
and our circuits use $10$ restarts; the other deep baselines use $3$
per configuration, selected by mean validation NLL. TCNet's selected
size differs per dataset.}
\label{tab:deep}
\end{table}

On Trivago, TCNet achieves the lowest test NLL, improving on
$\mathrm{Tree}_{63}$ by $0.023$ with $74$ times as many parameters.
On Expedia, every deep baseline has higher mean test NLL than MNL.

The distribution diagnostics distinguish the two settings: a
random-forest two-sample classifier achieves AUC $0.84$ on Expedia and $0.56$ on
Trivago, indicating stronger detectable temporal shift on Expedia
(Appendix~\ref{app:protocol}). Parameter count alone does not explain
the performance gap: circuits scaled to $44$k parameters maintain
similar test NLL, while TCNet and DeepHalo remain worse than MNL
when reduced to approximately $1.5$k--$1.6$k parameters
(Appendix~\ref{app:scale}). The circuit forms its intermediate
representations from the current menu's own features; this pattern is consistent with staged choice providing a useful
inductive bias under temporal shift. The trees are the only compared models to outperform both
classical baselines on both datasets, and $\mathrm{Tree}_{63}$
achieves the best result on Expedia with just $503$ parameters.

\subsection{The Consideration Circuit as an Output Head}
\label{sec:head}

Consideration circuits also integrate with deep choice models as
output heads. Section~\ref{sec:universality} shows that their
approximation limits depend on feature geometry, motivating an
encoder that learns the representations supplied to the circuit.
A feed-forward network, set encoder, or transformer produces a
feature vector for each item. We retain this encoder architecture
and replace its linear MNL readout with a consideration circuit,
training the encoder and circuit jointly. A schematic of this
encoder--circuit architecture is provided in
Appendix~\ref{app:head-matrix} (Figure~\ref{fig:pipeline}).
This composition separates feature processing from staged choice:
the encoder learns how item attributes are represented, while the
circuit models how those representations are combined through
successive rounds of consideration.

Table~\ref{tab:head} compares six encoders, including the identity,
with their native linear readout and with $\mathrm{Tree}_{63}$.
The circuit head lowers mean test NLL for every encoder on both
datasets. On Trivago, it further improves the strongest encoder,
TCNet, by $0.008$ to $2.4013$, the lowest Trivago NLL reported in
this paper. On Expedia, DeepHalo improves by $0.192$, while its
restart standard deviation decreases from $0.174$ to $0.012$,
bringing its mean test NLL below the MNL baseline.

We compare against
an MoE with the same number of leaves and an HME on the same tree.
These gated readouts perform similarly on Trivago, but
$\mathrm{Tree}_{63}$ achieves lower mean NLL than both for five of
the six encoders on Expedia. On raw features, their NLLs are
$2.699$ and $2.702$, compared with $2.598$ for the circuit; the full
comparison appears in Appendix~\ref{app:protocol}.
\ccmodel{} generalizes the MNL readout: it retains the same interface from item features to choice probabilities, includes a single MNL as a special case, and allows multiple stages of consideration.

\begin{table}[htbp]
\centering\small
\begin{tabular}{l cc cc}
\toprule
& \multicolumn{2}{c}{Expedia} & \multicolumn{2}{c}{Trivago}\\
\cmidrule(lr){2-3}\cmidrule(lr){4-5}
encoder & alone & $+\mathrm{Tree}_{63}$ & alone & $+\mathrm{Tree}_{63}$\\
\midrule
none       & \ms{2.6574}{.0008} & \ms{\bm{2.5983}}{.0018}
           & \ms{2.6400}{.0006} & \ms{2.4331}{.0023}\\
MLP(16)    & \ms{2.6798}{.0398} & \ms{2.6469}{.0152}
           & \ms{2.4809}{.0007} & \ms{2.4185}{.0027}\\
MLP(64,64) & \ms{2.7006}{.0401} & \ms{2.6621}{.0181}
           & \ms{2.4751}{.0010} & \ms{2.4148}{.0016}\\
DeepSets   & \ms{2.9113}{.2282} & \ms{2.8957}{.1787}
           & \ms{2.4179}{.0032} & \ms{2.4114}{.0118}\\
DeepHalo   & \ms{2.8272}{.1744} & \ms{2.6351}{.0117}
           & \ms{2.4151}{.0027} & \ms{2.4133}{.0024}\\
TCNet      & \ms{2.7377}{.0686} & \ms{2.6611}{.0322}
           & \ms{2.4093}{.0031} & \ms{\bm{2.4013}}{.0027}\\
\bottomrule
\end{tabular}
\caption{Encoders with their native linear readout (alone) vs.\ the
circuit head $\mathrm{Tree}_{63}$; test NLL, mean
{\scriptsize(sd)} over $10$ restarts; best per dataset in bold. On
Expedia the TCNet encoder uses its one-layer configuration; encoder
details in Appendix~\ref{app:protocol}.}
\label{tab:head}
\end{table}

\section{Conclusion}
\label{sec:conclusion}

Consideration circuits extend consider-then-choose to a differentiable, multistage process built from MNL units. We establish a sharp norm separation between depths 2 and 3, characterize the geometric conditions for depth-3 universality over deterministic choice tables, and prove optimal logarithmic norm scaling at depth 4. Compact circuits outperform the evaluated classical baselines on four fixed-pool benchmarks and achieve the lowest mean test NLL among the compared models on temporally shifted Expedia. As output heads, they lower mean test NLL for every tested encoder on both datasets, demonstrating the value of staged choice beyond a single MNL layer.

\bibliography{references}
\bibliographystyle{informs2014}

\begin{APPENDICES}
\renewcommand*{\theHsection}{appendix.\Alph{section}}
\section{Proof of Lemma~\ref{lem:consistency}}
\label{app:model-proofs}

\begin{proof}[Proof of Lemma~\ref{lem:consistency}]
At a source, both claims follow directly from
\eqref{eq:fwd-source}. Suppose they hold for every predecessor of an
interior node $v$. Since $q(v)$ is a probability vector,
\[
  p(v)=\sum_{u\in\Nin(v)}q(v)_u p(u)
\]
is a convex combination of probability vectors and is therefore itself a
probability vector. Moreover,
\[
  x(v)
  =\sum_{u\in\Nin(v)}q(v)_u x(u)
  =\sum_{u\in\Nin(v)}q(v)_u p(u)^\top\mathbf X
  =p(v)^\top\mathbf X.
\]
Topological induction proves the identity at every node. Finally, a
permutation of the rows of $\mathbf X$ permutes every source distribution
and leaves every source summary unchanged. The same is then true
successively at every interior node, so the sink distribution is permuted
in the same way as the menu.
\end{proof}

\section{Proofs for Section~\ref{sec:power}}
\label{app:power}

This appendix proves the theoretical results in Section~\ref{sec:power}. We retain the
notation of Definition~\ref{def:consideration-graphs}. When the labeled menu is
$S$, we write $p_v^S$ and $\mu_v^S$ for the distribution and feature
summary produced by node $v$; thus
\begin{equation*}
  \mu_v^S=\sum_{i\in S}p_v^S(i)x_i.
\end{equation*}
For an interior node, $q_v^S(u)$ denotes the routing probability assigned
to predecessor $u\in\Nin(v)$.
The set of source ancestors of $v$ is denoted by $\mathcal S(v)$,
with a source counted as its own ancestor.

\paragraph{Formal statements of the quantities of
Section~\ref{sec:power}.}
Cost is measured by
$\|\theta\|_{\rm taste}:=\max_{v\in V}\|w_v\|_2$, where $\theta$
collects all node tastes and edge biases; edge biases may be arbitrary
menu-independent reals and are not charged to this budget. For the
compromise task, writing $c_\triangle$ for the three choices of
\eqref{eq:compromise-task},
\begin{gather*}
  \err_\triangle(G,\theta)
  :=\max_{S\in\{\{1,3\},\{2,3\},\{1,2,3\}\}}
       \bigl[1-p^S(v^\star)_{c_\triangle(S)}\bigr],\\
  N_k^*(\epsilon)
  :=\inf\left\{\|\theta\|_{\rm taste}:
       \dep(G)\le k,\ \err_\triangle(G,\theta)\le\epsilon\right\},
\end{gather*}
the infimum over all finite graphs of the stated depth, with no
restriction on width or fan-in. Fix a labeled catalog
$X=(x_1,\ldots,x_N)\in(\R^d)^N$, write $[N]=\{1,\ldots,N\}$, and identify
a nonempty menu with its label set $S\subseteq[N]$. An observed menu family
is a nonempty set
$\mathcal F\subseteq 2^{[N]}\setminus\{\varnothing\}$. A deterministic
choice table on $\mathcal F$ is a map $c$ satisfying $c(S)\in S$ for every
$S\in\mathcal F$. When the graph is evaluated on
$X_S:=\{x_i:i\in S\}$, write $p^S(v)$ for the distribution produced at
node $v$, indexed by the labels in $S$ (the same object as $p_v^S$;
both forms are used), and define
\begin{equation*}
  \err_{\mathcal F}(G,\theta;c)
  :=\max_{S\in\mathcal F}\bigl[1-p^S(v^\star)_{c(S)}\bigr].
\end{equation*}

Two limits are built into the MNL unit itself. First, no linear score can
strictly expose a point contained in the convex hull of the other offered
features. Second, a softmax is invariant to a common shift of all its
scores. We therefore call $(X,\mathcal F)$ \emph{admissible} when it
satisfies both of the following conditions:
\begin{enumerate}[leftmargin=2.3em,itemsep=2pt]
  \item[\textnormal{(E)}] \emph{Local exposure.} For every
  $S\in\mathcal F$ and $i\in S$,
  \begin{equation*}
    x_i\notin\operatorname{conv}\{x_j:j\in S\setminus\{i\}\},
  \end{equation*}
  with the condition understood as vacuous when $S=\{i\}$.
  \item[\textnormal{(T)}] \emph{Translation separation.} There do not exist distinct
  non-singleton menus $A,B\in\mathcal F$, a bijection $\phi:A\to B$, and
  $t\in\R^d$ such that
  \begin{equation*}
    x_{\phi(i)}=x_i+t\qquad(i\in A).
  \end{equation*}
\end{enumerate}
The singleton exception in (T) is necessary because a singleton has only
one feasible choice. Conditions (E) and (T) are necessary for universality
over all deterministic choice tables at every finite depth. The results
below show that together they are sufficient.

\subsection{Two invariants used throughout}

\begin{lemma}[Source-mixture invariant]
\label{lem:app-source-mixture}
For every graph, node $v$, and menu $S$, there are coefficients
$\alpha_{vs}^S\ge0$, $s\in\mathcal S(v)$, summing to one, such that
\begin{equation}
  p_v^S=\sum_{s\in\mathcal S(v)}\alpha_{vs}^S p_s^S.
  \label{eq:app-source-mixture}
\end{equation}
Consequently, no node can assign an item more probability than the largest
probability assigned to that item by an upstream source.
\end{lemma}

\begin{proof}
The assertion is immediate at a source. If it holds at every predecessor of
an interior node $v$, substitute their source-mixture representations into
$p_v^S=\sum_{u\in\Nin(v)}q_v^S(u)p_u^S$. The new coefficients are sums of
products of nonnegative routing weights and old mixture coefficients;
they sum to one. Topological induction proves the claim.
\end{proof}

\begin{lemma}[Source-purity lower bound]
\label{lem:app-source-purity}
Suppose $|S|\ge2$ and the sink assigns item $i\in S$ probability at least
$1-\epsilon$, where $0<\epsilon<1/2$. Then some source $s$ satisfies
$p_s^S(i)\ge1-\epsilon$, and for every $j\in S\setminus\{i\}$,
\begin{equation}
  \|w_s\|_2
  \ge
  \frac{1}{\|x_i-x_j\|_2}
  \log\frac{1-\epsilon}{\epsilon}.
  \label{eq:app-source-purity-bound}
\end{equation}
In particular every graph attaining such a row obeys
$\|\theta\|_{\rm taste}=\Omega_X(\log(1/\epsilon))$.
\end{lemma}

\begin{proof}
\Eqref{eq:app-source-mixture} implies that at least one source has
$p_s^S(i)\ge1-\epsilon$. Since $p_s^S(j)\le\epsilon$, the source MNL odds
give
\[
  w_s^\top(x_i-x_j)
  =\log\frac{p_s^S(i)}{p_s^S(j)}
  \ge\log\frac{1-\epsilon}{\epsilon}.
\]
Cauchy--Schwarz yields \eqref{eq:app-source-purity-bound}.
\end{proof}

We shall also use a standard hard-to-soft argument, stated here with the
precise tie convention needed below. A \emph{hard report} at a source is its
unique utility maximizer. At an interior node, first replace every
predecessor summary by its hard reported item. The resulting
collection of hard reports, one per node and menu, is the \emph{hard
assignment} of the parameters. It is
\emph{locally output-strict} on a finite menu family if all score-maximizing
predecessors report the same item and every predecessor reporting a
different item has a strictly smaller score. Ties between predecessors that
report the same item do not change the hard output.

\begin{lemma}[Exponential softening]
\label{lem:app-softening}
Fix a finite graph, a finite menu family $\mathcal F$, and a locally
output-strict hard assignment that reports a table $c$. Multiply every
taste and every edge
bias in that assignment by $T$. There are constants $C,c_0>0$, depending
only on the fixed assignment, such that
\begin{equation}
  \err_{\mathcal F}(G,T\theta;c)\le Ce^{-c_0T}
  \label{eq:app-softening}
\end{equation}
for all sufficiently large $T$. In particular, the resulting taste budget
is $O(\log(1/\epsilon))$ at error $\epsilon$.
\end{lemma}

\begin{proof}
Choose $\delta>0$ no larger than every source winner-versus-loser gap
and every gap between a maximal hard routing score and the score of a
predecessor reporting a different item. Such a choice exists by finiteness
and local output-strictness; menus or nodes with no such comparison impose
no constraint on $\delta$. A source assigns total mass at most
$(N-1)e^{-T\delta}$ outside its hard item.

Suppose every predecessor of an interior node $v$ has mass at most
$\eta_T$ outside its own hard item, where $\eta_T$ decays exponentially.
Its summary is then within $D_X\eta_T$ of its hard feature, where
$D_X:=\max_{i,j}\|x_i-x_j\|_2$. Each routing logit under $T\theta$
therefore differs from $T$ times its unscaled hard routing score by at most
$T\|w_v\|_2D_X\eta_T$. Choose a predecessor
$u_\star$ with maximal hard routing score. For sufficiently large $T$,
its score exceeds that of every predecessor reporting a different item by
at least
\[
  T\delta-2T\|w_v\|_2D_X\eta_T\ge T\delta/2.
\]
The total routing mass on those predecessors is thus at most
$(|\Nin(v)|-1)e^{-T\delta/2}$. The mass outside $v$'s hard item is bounded
by this quantity plus $\eta_T$, so it also decays exponentially.
A finite topological induction proves \eqref{eq:app-softening}.
\end{proof}

\subsection{Affine changes of coordinates for a three-item menu}

Let $H_X:=\operatorname{span}\{x_2-x_1,x_3-x_1\}$. Non-collinearity means
$\dim H_X=2$. Fix the regular-triangle coordinates
\begin{equation}
  \bar x_1=(1,0),\qquad
  \bar x_2=(-1/2,\sqrt3/2),\qquad
  \bar x_3=(-1/2,-\sqrt3/2).
  \label{eq:app-regular-triangle}
\end{equation}

\begin{lemma}[Affine normalization]
\label{lem:app-affine-normalization}
There is an affine isomorphism from $x_1+H_X$ to $\R^2$ sending
$x_i$ to $\bar x_i$. Every taste may first be projected onto $H_X$ without
changing any choice probability or increasing its norm. The projected
tastes admit an invertible linear correspondence with tastes in $\R^2$
that preserves all choice probabilities under the affine map.
Corresponding taste norms differ by multiplicative constants depending
only on $(x_1,x_2,x_3)$.
\end{lemma}

\begin{proof}
There is a unique invertible linear map
$L:H_X\to\R^2$ satisfying
$L(x_i-x_1)=\bar x_i-\bar x_1$ for $i=2,3$. Set
$\Phi(x)=\bar x_1+L(x-x_1)$. For a transformed taste $a\in\R^2$, choose
$w\in H_X$ satisfying $L^\top a=w$. Then
$a^\top\Phi(x)=w^\top x+\gamma_a$ on $x_1+H_X$, where $\gamma_a$ is
constant across alternatives and therefore cancels from every softmax.
Convex combinations commute with $\Phi$, so the argument applies
inductively to all summaries. Conversely, an arbitrary ambient taste
$w\in\R^d$ may be replaced by its orthogonal projection onto $H_X$
without changing any probability: every menu item and every summary
lies in the affine plane $x_1+H_X$, so the component of $w$ orthogonal
to $H_X$ contributes the same constant to every score and cancels from
every softmax, while the projection does not increase $\|w\|_2$.
The map $L^\top:\R^2\to H_X$ is invertible, so there are constants
$0<c_X\le C_X<\infty$ such that, for every $a\in\R^2$,
\[
  c_X\|a\|_2\le\|L^\top a\|_2\le C_X\|a\|_2.
\]
\end{proof}

\subsection{Proof of Theorem~\ref{thm:separation}}
\label{app:separation-proof}

By Lemma~\ref{lem:app-affine-normalization}, it suffices to prove all rates
for the regular triangle \eqref{eq:app-regular-triangle}; constants change
only by factors depending on the original three features. For a taste $w$,
put $t_i=w^\top\bar x_i$ and
\begin{equation*}
  \operatorname{osc}(t):=\max_i t_i-\min_i t_i.
\end{equation*}
Since the triangle has diameter $\sqrt3$,
$\operatorname{osc}(t)\le\sqrt3\|w\|_2$. Conversely, after subtracting
the common mean from any prescribed score triple, the formula
\begin{equation}
  w=\frac23\sum_{i=1}^3t_i\bar x_i
  \label{eq:app-triangle-score-realization}
\end{equation}
realizes the centered scores, and gives
$\|w\|_2\le2\operatorname{osc}(t)$.

\paragraph{The MNL and random-utility floors.}
For a single MNL, let $a_i=e^{w^\top\bar x_i}$ and put
$r=a_1/a_3$, $s=a_2/a_3$. Its three target probabilities are
\[
  \frac{r}{1+r},\qquad \frac{s}{1+s},\qquad
  \frac{1}{1+r+s}.
\]
Writing $z=r+s$, at least one of the first two is no larger than
$z/(z+2)$. Therefore the largest possible minimum target probability is
at most
\[
  \max_{z>0}\min\left\{\frac{z}{z+2},\frac{1}{z+1}\right\}
  =\sqrt2-1,
\]
where the two terms agree at $z=\sqrt2$. Equality is attained at
$r=s=1/\sqrt2$. Hence plain MNL has the exact floor
$\epsilon\ge2-\sqrt2$.

A context-independent random-utility model is a probability distribution
over complete rankings, choosing from each menu the highest-ranked offered
item. Such a model satisfies regularity: if $S\subseteq T$ and $i\in S$,
then $P(i\mid T)\le P(i\mid S)$. In the compromise task,
\[
  P(3\mid123)\le P(3\mid23)=1-P(2\mid23)\le\epsilon.
\]
Accuracy also requires $P(3\mid123)\ge1-\epsilon$, and hence
$\epsilon\ge1/2$. This is tight: the equal mixture of the rankings
$1\succ2\succ3$ and $3\succ2\succ1$ gives target probability $1/2$ on
all three displayed menus.

\paragraph{A width-free depth-$2$ lower bound.}
Consider an arbitrary finite depth-$2$ star with sources
$q=1,\ldots,W$ and sink $v^\star$; this is no loss of generality,
since the only other graphs of depth at most $2$ are single vertices,
i.e.\ plain MNL, which the floor $\epsilon\ge2-\sqrt2$ already
excludes at these accuracies. On the triple menu write
\[
  p_q^{123}=(\xi_q,\upsilon_q,\zeta_q),
  \qquad \xi_q+\upsilon_q+\zeta_q=1.
\]
Source IIA gives
\begin{equation*}
  r_q^{13}:=p_q^{13}(1)=\frac{\xi_q}{\xi_q+\zeta_q},
  \qquad
  r_q^{23}:=p_q^{23}(2)=\frac{\upsilon_q}{\upsilon_q+\zeta_q}.
\end{equation*}
Let $t_i=w_{v^\star}^\top\bar x_i$, let
$R=\operatorname{osc}(t)$, and write the sink routing logit and weight as
\begin{equation*}
  g_q(S)=b_{qv^\star}+\sum_{i\in S}p_q^S(i)t_i,
  \qquad
  \gamma_q^S=\frac{e^{g_q(S)}}{\sum_a e^{g_a(S)}}.
\end{equation*}
The biases may be arbitrary.

Fix $0<\epsilon\le1/8$, set $\eta=2\epsilon$, and define source groups
\begin{align*}
  G_{13}&=\{q:r_q^{13}\ge1-\eta\},&
  G_{23}&=\{q:r_q^{23}\ge1-\eta\},&
  G_3&=\{q:\zeta_q\ge1-\eta\}.
\end{align*}
If the sink has error at most $\epsilon$, Markov's inequality applied to
the wrong-item mass gives
\begin{align}
  \gamma^{13}(G_{13}),\ \gamma^{23}(G_{23}),\
  \gamma^{123}(G_3)&\ge\frac12,
  \label{eq:app-good-group-mass}\\
  \gamma^{123}(G_{13}),\ \gamma^{123}(G_{23}),\
  \gamma^{13}(G_3),\ \gamma^{23}(G_3)&<2\epsilon,
  \label{eq:app-cross-group-mass}
\end{align}
where $\gamma^S(G)=\sum_{q\in G}\gamma_q^S$. For example, outside
$G_{13}$ each source has pair-menu wrong mass greater than $\eta$, so the
outside routing mass is at most $\epsilon/\eta=1/2$. A source in $G_{13}$
has $\zeta_q\le\eta$, so on the triple menu it creates at least
$1-\eta$ wrong mass and receives routing mass less than $2\epsilon$.
The other inequalities are identical; for $q\in G_3$ one uses
$r_q^{13},r_q^{23}\le\eta/(1-\eta)\le1/3$.

For a group $G$, put $U_G(S)=\sum_{q\in G}e^{g_q(S)}$. Global softmax
normalization cancels from group odds:
$\gamma^S(G)/\gamma^S(H)=U_G(S)/U_H(S)$. Deleting item $2$ from the
triple-menu source law gives
\[
  p_q^{123}=(1-\upsilon_q)p_q^{13}+\upsilon_q e_2,
\]
and therefore, with the edge bias cancelling,
\begin{equation}
  g_q(13)-g_q(123)
  =\upsilon_q\left[(t_1-t_2)
       +(1-r_q^{13})(t_3-t_1)\right].
  \label{eq:app-delete-two-score}
\end{equation}
Similarly,
\begin{equation}
  g_q(23)-g_q(123)
  =\xi_q\left[(t_2-t_1)
       +(1-r_q^{23})(t_3-t_2)\right].
  \label{eq:app-delete-one-score}
\end{equation}

If $t_1\le t_2$, \eqref{eq:app-delete-two-score} is at most
$\eta R$ on $G_{13}$ and has absolute value at most $\eta R$ on $G_3$.
Write $\ell(S)=\log[U_{G_{13}}(S)/U_{G_3}(S)]$. Termwise comparison gives
$\ell(13)-\ell(123)\le2\eta R$, while the group-mass bounds
\eqref{eq:app-good-group-mass}--\eqref{eq:app-cross-group-mass} give
\[
  2\log\frac1{4\epsilon}
  \le\ell(13)-\ell(123)\le2\eta R.
\]
If $t_1\ge t_2$, the identical argument uses
\eqref{eq:app-delete-one-score}, $G_{23}$, and $G_3$. In both cases,
\begin{equation*}
  2\epsilon R\ge\log\frac1{4\epsilon}.
\end{equation*}
Since $R\le\sqrt3\|w_{v^\star}\|_2$, this proves, uniformly over width and
all edge biases,
\begin{equation}
  N_2^*(\epsilon)
  \ge\frac{1}{2\sqrt3\,\epsilon}
       \log\frac1{4\epsilon}.
  \label{eq:app-depth2-norm-lower}
\end{equation}
A depth-$1$ graph can be embedded in a depth-$2$ star by adding a
single-predecessor sink, so the bound covers depth at most two.

\paragraph{A matching two-source construction.}
Let $0<\epsilon\le1/16$ and define
\[
  \delta=\frac\epsilon4,\qquad m=\frac\epsilon{64},\qquad
  M=\log\frac1m,\qquad
  L=\log\frac{1-\delta}{\delta},
\]
\begin{equation*}
  K=2\log\frac8\epsilon,\qquad
  V=\frac{2K}{\delta}=\frac{16}{\epsilon}\log\frac8\epsilon.
\end{equation*}
Use two sources with item-score triples
\begin{equation*}
  s_A=(0,-M,L),\qquad s_C=(0,M,-M),
\end{equation*}
and a zero-bias sink with item scores $t=(V,K,0)$. Formula
\eqref{eq:app-triangle-score-realization} realizes all three score triples
by tastes. Up to normalization, the source item weights are
\[
  (1,m,(1-\delta)/\delta)
  \quad\text{and}\quad (m,1,m^2).
\]
The intended sink winner on $13,23,123$ is respectively $C,C,A$.
Their exact score margins are
\begin{align*}
  \Delta_{13}
  &=\frac{V}{1+m}-V\delta,\notag\\
  \Delta_{23}
  &=\frac{K}{1+m^2}
    -\frac{Km}{m+(1-\delta)/\delta},\notag\\
  \Delta_{123}
  &=\frac{2K+Km\delta}{1+m\delta}
    -\frac{2Km/\delta+K}{1+m+m^2}.
\end{align*}
For $\epsilon\le1/16$, the first two exceed $K/2$. For the last,
$2m/\delta=1/8$, so its second term is at most $9K/8$, whereas
$m\delta\le1/1024$ makes its first term larger than $13K/8$; hence it too
exceeds $K/2$. The wrong sink-route probability is therefore at most
$e^{-K/2}=\epsilon/8$. The selected-source errors on $13,23,123$ are at
most $m,m^2,2\delta$, respectively, so total error is below $\epsilon$.
The two source score ranges are $O(\log(1/\epsilon))$ and the sink range is
$V$. \Eqref{eq:app-triangle-score-realization} therefore gives
\begin{equation}
  N_2^*(\epsilon)
  \le C\frac{\log(1/\epsilon)}{\epsilon}
  \label{eq:app-depth2-upper}
\end{equation}
for an absolute constant $C$ on the regular triangle.

\paragraph{The five-node depth-$3$ construction.}
Use the graph in Figure~\ref{fig:five-node-hard}. Its three source rankings
are
\[
  1\succ2\succ3,\qquad
  2\succ1\succ3,\qquad
  3\succ1\succ2.
\]
The middle node receives the last two sources and ranks
$1\succ3\succ2$; the sink receives the first source and the middle node
and ranks $2\succ3\succ1$. Reading the highest-ranked reported item gives
exactly $13\mapsto1$, $23\mapsto2$, and $123\mapsto3$. Give every node
the appropriate permutation of the score triple $(2,1,0)$ and set all
edge biases to zero. In this hard assignment, every source winner gap
and every gap from a maximal routing score to a different report is at
least $1$. Scaling all tastes by $T$ makes these hard gaps at least $T$.
Lemma~\ref{lem:app-softening} therefore gives error at most
$Ce^{-cT}$, while \eqref{eq:app-triangle-score-realization} gives taste
budget $O(T)$. Taking $T=C_0\log(C_0/\epsilon)$ proves
$N_3^*(\epsilon)=O(\log(1/\epsilon))$.

Finally, Lemma~\ref{lem:app-source-purity}, applied to the triple-menu
target, gives $N_3^*(\epsilon)=\Omega_X(\log(1/\epsilon))$. Together with
\eqref{eq:app-depth2-norm-lower} and \eqref{eq:app-depth2-upper}, and then
undoing the affine normalization, this proves
Theorem~\ref{thm:separation}.

\subsection{Analytic lemmas for the universality theorems}
\label{app:general-universality}

For a source with taste $u$ and a labeled menu $S$, define its partition
function and mean feature by
\begin{equation}
  Z_S(u):=\sum_{i\in S}e^{u^\top x_i},
  \qquad
  \mu_S(u):=\nabla\log Z_S(u)
  =\sum_{i\in S}\frac{e^{u^\top x_i}}{Z_S(u)}x_i.
  \label{eq:app-partition-mean}
\end{equation}
Thus a source $s$ with taste $w_s=u$ has summary
$\mu_s^S=\mu_S(u)$.  We first record a simple error estimate that will
be used repeatedly.

\begin{lemma}[Mixture-error bookkeeping]
\label{lem:app-mixture-error}
Fix a menu $S$ and item $i\in S$.  If an interior node gives total routing
mass at least $1-\delta$ to predecessors $u$ satisfying
$p_u^S(i)\ge1-\eta$, then
\begin{equation*}
  1-p_v^S(i)\le\delta+\eta.
\end{equation*}
Consequently, on a fixed finite graph, exponentially small source errors
and exponentially small routing mass on wrong-report branches give an
exponentially small sink error.
\end{lemma}

\begin{proof}
The defining mixture gives
$p_v^S(i)\ge(1-\delta)(1-\eta)\ge1-\delta-\eta$.
Iterating this inequality along a topological ordering adds only finitely
many exponentially small terms.
\end{proof}

\begin{proposition}[Necessity of (E) and (T)]
\label{prop:app-admissibility-necessary}
If (E) fails, some deterministic table has error at least $1/2$ under
every finite graph and every parameter choice.  If (E) holds but (T)
fails, the same conclusion holds.  Thus (E)--(T) are necessary for
all-table approximation to arbitrary accuracy.
\end{proposition}

\begin{proof}
Suppose first that
$x_i\in\operatorname{conv}\{x_j:j\in S\setminus\{i\}\}$ for an observed
menu $S$.  Write $x_i=\sum_{j\ne i}\lambda_jx_j$, where
$\lambda_j\ge0$ and $\sum_j\lambda_j=1$.  For every source taste $w$,
weighted AM--GM gives
\[
  e^{w^\top x_i}
  =\prod_{j\ne i}\bigl(e^{w^\top x_j}\bigr)^{\lambda_j}
  \le\sum_{j\ne i}\lambda_je^{w^\top x_j}
  \le\sum_{j\ne i}e^{w^\top x_j}.
\]
Hence every source assigns probability at most $1/2$ to $i$ on $S$.
Lemma~\ref{lem:app-source-mixture} propagates the same cap to the sink,
so the row $c(S)=i$ cannot be approximated below error $1/2$.

Now suppose distinct non-singleton menus $A,B$ violate (T), with bijection
$\phi:A\to B$ and translation $t$.  We claim, by induction over the graph,
that every node satisfies
\begin{equation}
  p_v^B(\phi(i))=p_v^A(i)\quad(i\in A),
  \qquad \mu_v^B=\mu_v^A+t.
  \label{eq:app-translation-equivariance}
\end{equation}
At a source, translation adds the common scalar $w_v^\top t$ to all item
scores.  If the identities hold for every predecessor $u$ of an interior
node $v$, then its incoming score on $B$ is
\[
  w_v^\top\mu_u^B+b_{uv}
  =w_v^\top\mu_u^A+b_{uv}+w_v^\top t.
\]
The last term is common to all predecessors, so their routing weights are
the same on $A$ and $B$, proving \eqref{eq:app-translation-equivariance}.
Choose distinct $i,k\in A$, set $c(A)=i$ and $c(B)=\phi(k)$, and complete
the table arbitrarily.  The two required probabilities are distinct
coordinates of the same probability vector, so one is at most $1/2$.
\end{proof}

The next lemma identifies exactly why condition (T), rather than a
stronger generic-position assumption, is the relevant obstruction.

\begin{lemma}[Gibbs means characterize translations]
\label{lem:app-translation-characterization}
Let the feature points within each of two menus $A,B$ be distinct.  If
$\mu_A(u)-\mu_B(u)$ is constant on a nonempty open subset of $\R^d$, then
the two feature supports are exact translates.  Conversely, exact
translation makes this mean difference constant.
\end{lemma}

\begin{proof}
The functions in \eqref{eq:app-partition-mean} are real analytic, so an
identity on an open set holds everywhere.  If
$\mu_A(u)-\mu_B(u)\equiv t$, integration gives
\[
  \log Z_A(u)-\log Z_B(u)=t^\top u+c,
  \qquad Z_A(u)=e^c e^{t^\top u}Z_B(u).
\]
Finite multivariate exponential functions with distinct exponent vectors
are linearly independent.  To see this, restrict $u$ to a generic line on
which all exponent vectors have distinct projections, divide a putative
linear relation by its largest exponential, let the line parameter tend to
$+\infty$, and iterate.  Comparing the exponent supports and their unit
coefficients gives $e^c=1$ and
$\{x_i:i\in A\}=t+\{x_j:j\in B\}$.  The converse follows directly from
the two partition functions.
\end{proof}

\begin{lemma}[An anchored generic probe]
\label{lem:app-anchored-probe}
Assume (E)--(T), and let $A_0$ be a non-singleton observed menu with a
prescribed item $i_0\in A_0$.  There is $v\in\R^d$ such that
\begin{enumerate}[leftmargin=2em]
  \item $i_0$ is the unique $v$-worst item of $A_0$;
  \item every observed menu $S$ has a unique $v$-worst item
  \begin{equation*}
    y_S:=\argmin_{i\in S}v^\top x_i;
  \end{equation*}
  \item for every distinct observed pair $A,S$ with at least one member
  non-singleton, the scalar function
  \begin{equation*}
    \psi_{AS}^v(u):=v^\top\bigl(\mu_A(u)-\mu_S(u)\bigr)
  \end{equation*}
  is nonconstant.
\end{enumerate}
Without the anchoring requirement, such a $v$ exists whenever a
non-singleton menu is observed.
\end{lemma}

\begin{proof}
By (E), there is a nonempty open cone of vectors making $i_0$ the strict
minimizer on $A_0$.  For a relevant pair $(A,S)$, define
\[
  K_{AS}:=\left\{v:
  v^\top\bigl[(\mu_A(u)-\mu_S(u))
  -(\mu_A(0)-\mu_S(0))\bigr]=0\ \text{for every }u\right\}.
\]
This is a linear subspace.  It is proper: otherwise
$\mu_A-\mu_S$ would be constant, contradicting
Lemma~\ref{lem:app-translation-characterization} and (T); when exactly one
menu is a singleton, unequal cardinalities already preclude translation.
Remove from the anchoring cone the finitely many proper subspaces $K_{AS}$
and the finitely many hyperplanes $v^\top(x_i-x_j)=0$ for distinct
features co-occurring in an observed menu.  A finite union of proper
subspaces and hyperplanes cannot cover a nonempty open set.  Any remaining
$v$ has all three properties.  The unanchored claim is the same argument
starting from $\R^d$.
\end{proof}

For each observed $A$ and $i\in A$, (E) gives a direction $a_{A,i}$ for
which $i$ uniquely maximizes $a_{A,i}^\top x_j$ over $j\in A$.
Perturbing it within that open exposing cone avoids every feature-tie
hyperplane arising in the finite observed family.  We therefore define its
unique hard trace by
\begin{equation}
  q_{A,i}(S):=\argmax_{j\in S}a_{A,i}^\top x_j.
  \label{eq:app-hard-trace}
\end{equation}

\begin{lemma}[Uniform exponential noncollision on a common ray]
\label{lem:app-exponential-noncollision}
Fix $v$ from Lemma~\ref{lem:app-anchored-probe} and the finitely many
directions $a_{A,i}$.  There are a fixed vector $\rho\in\R^d$ and
constants $c_0,\kappa_0,T_0>0$ such that, simultaneously for every
non-singleton observed $A$, every $i\in A\setminus\{y_A\}$, and every
observed $S\ne A$, the exact gains
\begin{equation}
  g_{A,i}^T(S)
  :=v^\top\bigl(\mu_S(Ta_{A,i})-\mu_S(-Tv+\rho)\bigr)
  \label{eq:app-exact-gain}
\end{equation}
satisfy
\begin{equation}
  \left|g_{A,i}^T(A)-g_{A,i}^T(S)\right|
  \ge c_0e^{-\kappa_0T}\qquad(T\ge T_0).
  \label{eq:app-exponential-gap}
\end{equation}
Moreover,
\begin{equation}
  g_{A,i}^T(A)\longrightarrow v^\top(x_i-x_{y_A})>0,
  \label{eq:app-own-gain-limit}
\end{equation}
and both source laws in \eqref{eq:app-exact-gain} concentrate
exponentially on their unique limiting winners.
\end{lemma}

\begin{proof}
The gain difference is the exact identity
\[
  g_{A,i}^T(A)-g_{A,i}^T(S)
  =\psi_{AS}^v(Ta_{A,i})-\psi_{AS}^v(-Tv+\rho).
\]
After multiplication by the four positive partition-function
denominators, its numerator has the form
\begin{equation}
  N_{A,i,S}(T,\rho)
  =\sum_{\lambda\in\Lambda_{A,i,S}}
     C_{A,i,S,\lambda}(\rho)e^{\lambda T},
  \label{eq:app-exponential-polynomial}
\end{equation}
where $\Lambda_{A,i,S}\subset\R$ is finite and every coefficient is a
finite real exponential sum in $\rho$.  These coefficients are not all
identically zero.  Indeed, at $T=0$ the cross-multiplied numerator is that
of $\psi_{AS}^v(0)-\psi_{AS}^v(\rho)$, which is not the zero function
because $\psi_{AS}^v$ is nonconstant.

For each of the finitely many triples, choose the largest exponent in
\eqref{eq:app-exponential-polynomial} whose coefficient is not identically
zero.  A nonzero real-analytic function has a zero set with empty interior,
so one common finite $\rho$ can be chosen outside the finite union of the
selected zero sets.  With this $\rho$, each selected leading term
eventually dominates all lower-exponent terms.  The positive denominators
grow at most exponentially in $T$.  Taking a finite minimum of numerator
constants and a finite maximum of denominator exponents proves
\eqref{eq:app-exponential-gap}.  Finally, the source directions have
strict winner gaps on every observed menu, so a winner separated by score
gap $\delta T-O(1)$ has total losing mass at most
$(N-1)e^{-\delta T+O(1)}$.  This proves the concentration statement and
\eqref{eq:app-own-gain-limit}.
\end{proof}

\subsection{Proof of Theorem~\ref{thm:universality3}}
\label{app:proof-universality3}

Let $K_N:=2^N-N-1$, the number of non-singleton subsets of $[N]$.
For $N=1$, the one-source graph suffices.  A family containing only
singletons is also immediate, since every node necessarily reports the
only offered item.  Suppose therefore that a non-singleton menu is
observed.  Choose one such menu $A_0$ and apply
Lemma~\ref{lem:app-anchored-probe} with $i_0=c(A_0)$, so that
$c(A_0)=y_{A_0}$.  Call a row $A$ \emph{active} when $c(A)\ne y_A$.
Singletons and the anchor are inactive, so there are at most $K_N-1$
active rows.

The topology is chosen before the catalog, family, and table.  It contains
one common source $R$, one unary relay $R'$, one sink, and $K_N-1$
identical two-node slots.  A slot consists of a foreground source $Q$ and
a binary interior node $H$ receiving $Q$ and $R$; the sink receives $R'$
and all slot outputs, while $R'$ receives only $R$.  Hence the topology has
depth at most $3$ and
\begin{equation*}
  2(K_N-1)+3=2^{N+1}-2N-1\le2^{N+1}
\end{equation*}
nodes.  After the table is known, assign one anonymous slot to every active
row and suppress all unused slots at the sink.  This changes parameters,
not topology.

Fix a large $T$.  Give $R$ taste $-Tv+\rho$, where $\rho$ is from
Lemma~\ref{lem:app-exponential-noncollision}; the unary relay has the same
distribution and summary.  For the slot assigned to $(A,i)$, give its
foreground source taste $Ta_{A,i}$ and abbreviate
\[
  g(S):=g_{A,i}^T(S),\qquad g_A:=g(A).
\]
For all sufficiently large $T$, $g_A>0$ and $g(S)\ne g_A$ for every
$S\ne A$.  Partition the off-rows as
\[
  L_A:=\{S\ne A:g(S)<g_A\},
  \qquad U_A:=\{S\ne A:g(S)>g_A\}.
\]
If $U_A\ne\varnothing$, set
\begin{equation}
  \tau_A:=\frac12\left(g_A+\min_{S\in U_A}g(S)\right);
  \label{eq:app-depth3-threshold}
\end{equation}
if $U_A=\varnothing$, set $\tau_A=3g_A/2$.  With an empty minimum
interpreted as $+\infty$, put
\begin{equation*}
  m_A:=\min\left\{g_A,
       \min_{S\in L_A}[g_A-g(S)]\right\},
  \qquad \eta_A:=m_A/4.
\end{equation*}

Give the hidden node taste $-M_Av$, bias $M_A\tau_A$ on its foreground
edge, and bias zero on its $R$ edge.  Its exact foreground routing
probability is
\begin{equation*}
  \alpha_A^S
  =\frac{1}{1+\exp[-M_A(\tau_A-g(S))]}.
\end{equation*}
Give the sink taste $Kv$, bias $K(-g_A+\eta_A)$ on the edge from this
hidden node, and bias zero on the relay edge.  Since
\[
  \mu_{H_A}^S-\mu_{R'}^S
  =\alpha_A^S(\mu_{Q_A}^S-\mu_R^S),
\]
the hidden-versus-relay sink score, divided by $K$, is exactly
\begin{equation}
  D_A(S)=-g_A+\eta_A+\alpha_A^Sg(S).
  \label{eq:app-depth3-root-score}
\end{equation}
For fixed $T$, as $M_A\to\infty$, $D_A(A)\to\eta_A>0$.
On $L_A$, the limit is $-g_A+\eta_A+g(S)\le-3m_A/4$;
on $U_A$, it is $-g_A+\eta_A\le-3m_A/4$.
Thus in the hard-routing limit the module is above the common baseline
on its own row and below it on every other observed row.

The midpoint in \eqref{eq:app-depth3-threshold}, together with
\eqref{eq:app-exponential-gap}, gives common constants $c_1>0$ and
$\kappa_0<\infty$ such that
\begin{equation*}
  m_A\ge c_1e^{-\kappa_0T},
  \qquad
  \min_{S\in\mathcal F}|\tau_A-g(S)|
  \ge c_1e^{-\kappa_0T}
\end{equation*}
for every active slot and all large $T$.  Choose
\begin{equation*}
  M_A=C_1Te^{\kappa_0T},
  \qquad K=C_2Te^{\kappa_0T},
\end{equation*}
with sufficiently large fixed $C_1,C_2$ and $C_1c_1>\kappa_0$.
The wrong hidden-routing mass is then $O(e^{-C_1c_1T})$.
Since the gains are uniformly bounded, its perturbation of
\eqref{eq:app-depth3-root-score} is $o(m_A)$, and every relevant sink
comparison has logit margin at least $cT$ for some fixed $c>0$.  Give each unused slot the sink-edge bias $-K B_{\rm off}$, where
$B_{\rm off}:=1+\max_i v^\top x_i-\min_i v^\top x_i$,
and set its dedicated source taste, hidden-node taste, and incoming
biases to zero. Every summary lies in the catalog convex hull, so each
unused slot scores at least $K$ below the relay at the sink.
Lemma~\ref{lem:app-mixture-error}, source concentration, and finiteness
of the graph now give
\begin{equation*}
  \err_{\mathcal F}(G_N^{(3)},\theta_T;c)\le C_3e^{-c_3T}.
\end{equation*}
Indeed, on an active row its own hidden node is exponentially close to the
foreground source, while every other module loses to the relay; on a
baseline row every active module loses to $R'$, whose source is
exponentially concentrated on $y_S=c(S)$.

The source tastes are $O(T)$ and the hidden and sink tastes are
$O(Te^{\kappa_0T})$; all edge biases used above have the same order.
Taking $T=c_3^{-1}\log(C_3/\epsilon)$ and renaming constants proves
\[
  \|\theta_T\|_{\rm taste}
  \le C_X\epsilon^{-\kappa}\log\frac{C_X}{\epsilon}.
\]
For a fixed catalog there are only finitely many labeled menu families,
anchors, and deterministic tables. Taking maxima over the admissible
instances makes $C_X,\kappa$ independent of $\mathcal F$ and $c$.
For sufficiently small $\epsilon$, the constant inside the logarithm
can be absorbed into its prefactor, giving the stated
$C_X\epsilon^{-\kappa}\log(1/\epsilon)$ bound.
For $N\ge2$, choose a fixed $\epsilon_0<1-1/N$ in this regime.
The parameters attaining $\epsilon_0$ cover the compact interval
$[\epsilon_0,1-1/N]$, where $\epsilon^{-\kappa}\log(1/\epsilon)$
has a positive minimum, after enlarging $C_X$.
For $1-1/N\le\epsilon<1$, zero tastes make every source and hence
the sink uniform on each menu, giving error at most $1-1/N$.
The case $N=1$ is exact with zero tastes. Thus the stated norm bound
holds throughout $0<\epsilon<1$.

\subsection{Proof of Proposition~\ref{prop:onehot}}
\label{app:onehot-proof}

We first record exactly where affine independence is used. Recall that
$x_1,\ldots,x_N$ are affinely independent when the $N-1$ vectors
$x_i-x_N$, $i<N$, are linearly independent.

\begin{lemma}[Realizing arbitrary item utilities]
\label{lem:app-affine-utilities}
Assume that $x_1,\ldots,x_N$ are affinely independent.
For every $a=(a_1,\ldots,a_N)\in\R^N$, there are $w\in\R^d$ and a
constant $\varkappa$ such that
\begin{equation}
  w^\top x_i=a_i+\varkappa\qquad(i\in[N]).
  \label{eq:app-affine-realization}
\end{equation}
Moreover, $w$ can be chosen with
\begin{equation}
  \|w\|_2\le K_X\bigl(\max_i a_i-\min_i a_i\bigr)
  \label{eq:app-affine-utility-norm}
\end{equation}
for a finite constant $K_X$ depending only on the catalog.
\end{lemma}

\begin{proof}
Solve
$w^\top(x_i-x_N)=a_i-a_N$, $i<N$. The coefficient map has full row rank,
so it has a bounded right inverse; this proves
\eqref{eq:app-affine-utility-norm}. Taking
$\varkappa=w^\top x_N-a_N$ proves \eqref{eq:app-affine-realization}.
\end{proof}

Thus any complete strict ranking, and any prescribed finite vector of item
utilities at an interior node, is realizable. Affine independence also
implies admissibility. Condition (E) follows because any affine dependence
inside a subcatalog would extend to the whole catalog. For (T), if
$x_{\phi(i)}-x_i=t$ on a non-singleton menu, then for distinct $i,j$,
\[
  x_{\phi(i)}-x_i-x_{\phi(j)}+x_j=0.
\]
Affine independence forces the corresponding formal coefficients to vanish.
Injectivity of $\phi$ then gives $\phi(i)=i$ for every $i$, so $t=0$ and
the two labeled menus coincide.

Fix a strict reference order $\rho$ on $[N]$. Let
$\tau(i)\in\{0,\ldots,N-1\}$ be its rank, with a larger number meaning a
better item, and define the baseline rule
\begin{equation*}
  r(S):=\operatorname*{arg\,max}_{i\in S}\tau(i).
\end{equation*}

\paragraph{A terminal-chain primitive.}
Fix a non-singleton menu $A$ and
$i\in A\setminus\{r(A)\}$. Partition the remaining items as
\begin{align*}
  L&=\{\ell\in A:\tau(\ell)<\tau(i)\},\\
  H&=\{h_1,\ldots,h_q\},
  \qquad \tau(i)<\tau(h_1)<\cdots<\tau(h_q),
\end{align*}
so $h_q=r(A)$, and set
\begin{equation*}
  P_j:=\{i\}\cup L\cup\{h_1,\ldots,h_j\},
  \qquad j=1,\ldots,q.
\end{equation*}

\begin{lemma}[Baseline-plus-terminal-chain primitive]
\label{lem:app-terminal-primitive}
There is a depth-$2$ star $D_{A,i}$ with exactly $|A|$ source
predecessors whose hard output on every nonempty $S\subseteq[N]$ is
\begin{equation}
  D_{A,i}(S)=
  \begin{cases}
    i,&r(S)=i\text{ or }S=P_j\text{ for some }j\le q,\\
    r(S),&\text{otherwise}.
  \end{cases}
  \label{eq:app-terminal-behavior}
\end{equation}
Every source has winner gap at least one. At the center, each
maximizing predecessor exceeds every predecessor reporting a different
item by at least one. One source is the common baseline-ranking source
for every pair $(A,i)$.
\end{lemma}

\begin{proof}
Let $O=[N]\setminus A$, let
\begin{equation*}
  u_a=2(\tau(a)-\tau(i))\quad(a\in A),
  \qquad u_o=-10N\quad(o\in O),
\end{equation*}
and put $B=2N+2$. Lemma~\ref{lem:app-affine-utilities} realizes the
$u_j$ as utilities at the center of the star. Within a displayed block,
$\downarrow$ means decreasing $\rho$-rank. The source rankings and their
incoming edge biases are:
\begin{enumerate}[leftmargin=2em,itemsep=2pt]
  \item a main source $M$, bias $0$, ranking
  $O^\downarrow\succ i\succ(A\setminus\{i\})^\downarrow$;
  \item the common baseline source $R_\rho$, bias $-B$, ranking all items
  by decreasing $\rho$-rank;
  \item for every $\ell\in L$, a lower gate $G_\ell$, bias $-1$, ranking
  \[
    O^\downarrow\succ\ell\succ H^\downarrow
    \succ i\succ(L\setminus\{\ell\})^\downarrow;
  \]
  \item for every $1\le t<q$, a higher gate $G_t$, bias
  $-u_{h_t}-1$, ranking
  \[
    O^\downarrow\succ h_t\succ h_q\succ h_{q-1}\succ\cdots
    \succ h_{t+1}\succ i
    \succ(L\cup\{h_1,\ldots,h_{t-1}\})^\downarrow.
  \]
\end{enumerate}
Assign utilities $N-1,N-2,\ldots,0$ in the order of each source
ranking and realize them using Lemma~\ref{lem:app-affine-utilities}.
This gives source winner gaps at least one; use the same realizing
taste for the baseline ranking in every primitive. The lists are
complete strict rankings and their number is $1+1+|L|+(q-1)=|A|$.

If $S$ contains an outsider, the main source and all gates report the
highest-$\rho$ outsider. If it is not $r(S)$, the baseline reports
$r(S)\in A$ and scores at least
$-2(N-1)-(2N+2)>-10N$, thereby restoring the baseline rule. If
$S\subseteq A$ but $i\notin S$, the main source reports $r(S)$ with score
$u_{r(S)}$; every gate reports an item of no larger utility and pays bias at
most $-1$, while the baseline pays $-B$.

Now suppose $i\in S\subseteq A$. If $S\cap H=\varnothing$, the main
source reports $i$ with score zero. A present lower gate reports $\ell$
with score at most $-3$, and every other gate reporting a different item
has negative score. Finally suppose $r(S)=h_j$. If a required
$\ell\in L$ is missing, $G_\ell$ reports $h_j$ with score
$u_{h_j}-1\ge1$. If a required $h_t$, $t<j$, is missing, $G_t$ reports
$h_j$ with score $u_{h_j}-u_{h_t}-1\ge1$. Thus any missing chain item
restores $h_j=r(S)$. If none is missing, then $S=P_j$: the main source
reports $i$ with score zero, every present gate scores at most $-1$, every
later gate that also reports $i$ has negative bias, and the baseline scores
at most $2(N-1)-(2N+2)=-4$. Hence the main source wins. These cases prove
\eqref{eq:app-terminal-behavior} and the unit margin.
\end{proof}

\paragraph{The compressed topology and its count.}
Before seeing the features, menu family, table, or accuracy, construct one
common source $R_\rho$. For every non-singleton abstract menu
$A\subseteq[N]$, preallocate one interior slot $D_A$, fed by $R_\rho$ and
$|A|-1$ dedicated sources. Connect $R_\rho$ and every $D_A$ to one common
sink. The longest path has three vertices and
\begin{equation*}
  |V|
  =2+\sum_{k=2}^N k\binom Nk
  =N2^{N-1}-N+2.
\end{equation*}
Sharing $R_\rho$ is valid because its taste is identical in every
primitive, while every outgoing edge retains its own bias.

Given a table $c$, if $c(A)=i\ne r(A)$, parameterize slot $D_A$ as
$D_{A,i}$; if $c(A)=r(A)$, mark its sink edge inactive, copy the common
baseline taste to all its dedicated sources, and give its incoming edges
equal biases.  The whole inactive slot then reports the baseline law on
every menu.  At the sink use
item utilities
\begin{equation*}
  v_j=-2\tau(j),
\end{equation*}
give the baseline edge bias zero, give active slot $A$ the bias
\begin{equation*}
  \beta_A=v_{r(A)}+1-v_{c(A)}\le-1,
\end{equation*}
and give every inactive slot bias $-2N$.

On a runtime menu $S$, the baseline score is $v_{r(S)}$. Every inactive
slot reports $r(S)$ with sink score $v_{r(S)}-2N$. An active slot either
reports $r(S)$ or differs only when $S=P_j$ in its terminal chain. In the
latter case its sink score is
$v_{c(A)}+\beta_A=v_{r(A)}+1$. For a false prefix $j<q$,
$v_{h_j}\ge v_{r(A)}+2$, so the baseline wins by at least one. At the
endpoint $P_q=A$, the slot wins by exactly one. Thus the hard sink outputs
$c(S)$ on every menu, with a gap of at least one from a maximizing
predecessor to every predecessor reporting another item. A table observed
only on $\mathcal F$ may be extended arbitrarily to the other menus, so the
same topology works for every family.

Lemma~\ref{lem:app-softening}, Lemma~\ref{lem:app-affine-utilities}, and
the bounded integer utilities above now give, for sufficiently small
$\epsilon$,
\[
  \err_{\mathcal F}(G,\theta;c)\le\epsilon,
  \qquad
  \|\theta\|_{\rm taste}
  \le C_X\log\frac{C_X}{\epsilon}.
\]
Absorbing the constant inside the logarithm into its prefactor gives
the bound in Proposition~\ref{prop:onehot} for sufficiently small
$\epsilon$. The zero-taste and compact-interval argument in the proof
of Theorem~\ref{thm:universality3} extends it to every $0<\epsilon<1$.
There are finitely many completed tables, so $C_X$ can be chosen to
depend only on the catalog. The displayed node count is an explicit
upper bound; no minimality claim is made.

\subsection{A Superlogarithmic Norm Lower Bound at Depth Three}
\label{app:superlog-proof}

For an edge $(u,v)$, define its routing logit and routing probability on
menu $S$ by
\begin{equation*}
  g_{uv}^S:=b_{uv}+w_v^\top\mu_u^S,
  \qquad
  \alpha_{uv}^S
  :=\frac{e^{g_{uv}^S}}
          {\sum_{a\in\Nin(v)}e^{g_{av}^S}}.
\end{equation*}
The construction begins with five points
\begin{equation*}
  b=(-1,1),\quad x=(1,1),\quad
  \ell=(-3,-1),\quad r=(3,-1),\quad z=(0,-5)
\end{equation*}
and four menu--target pairs
\begin{align*}
  A&=\{b,x,\ell,r\}\mapsto x,&
  B_\ell&=A\setminus\{\ell\}\mapsto b,\notag\\
  B_r&=A\setminus\{r\}\mapsto b,&
  U&=A\cup\{z\}\mapsto b.
\end{align*}
The cyclic orders $b,x,r,\ell$ and $b,x,r,z,\ell$ have consecutive cross
products
\[
  -4,-12,-12,-4
  \quad\text{and}\quad
  -4,-14,-24,-14,-4,
\]
so $A$ and $U$ are strictly convex and the deletion menus are
nondegenerate triangles. Thus (E) holds. The only equal-cardinality pair is
$B_\ell,B_r$. Their centroids are $(1,1/3)$ and $(-1,1/3)$, so the
only candidate translation is $(-2,0)$; translating the triangle
$B_\ell$ by $(-2,0)$ does not produce $B_r$, so
they are not translates. Their diameters equal $\sqrt{20}$, while
\begin{equation*}
  \operatorname{diam}(A)=6,
  \qquad \operatorname{diam}(U)=\sqrt{37}.
\end{equation*}

The persistence mechanism uses the identity
\begin{equation}
  x-z=\frac12(x-b)+(x-\ell)+2(x-r).
  \label{eq:app-chord-identity}
\end{equation}
For any MNL source $R$, conditioning its Gibbs law before and after a single
item is deleted or added also gives numbers
$\lambda_\ell,\lambda_r,\eta\in(0,1)$ satisfying
\begin{align}
  \mu_R^A
  &=(1-\lambda_\ell)\mu_R^{B_\ell}+\lambda_\ell\ell,
  \label{eq:app-cond-left}\\
  \mu_R^A
  &=(1-\lambda_r)\mu_R^{B_r}+\lambda_r r,
  \label{eq:app-cond-right}\\
  \mu_R^U
  &=(1-\eta)\mu_R^A+\eta z.
  \label{eq:app-cond-up}
\end{align}

\begin{lemma}[Source persistence]
\label{lem:app-source-persistence}
Let $T_n\to\infty$. If sources $Q_n$ satisfy
$1-p_{Q_n}^A(x)=O(e^{-T_n})$, then the same bound holds on
$B_\ell,B_r,U$, and
\begin{equation}
  \|\mu_{Q_n}^A-\mu_{Q_n}^S\|_2=O(e^{-T_n})
  \quad(S\in\{B_\ell,B_r,U\}).
  \label{eq:app-source-persistent-means}
\end{equation}
\end{lemma}

\begin{proof}
Deletion can only increase the source probability of $x$. For the addition
menu, put
\[
  R_y:=\exp\{w_{Q_n}^\top(y-x)\},
  \qquad E:=R_b+R_\ell+R_r.
\]
Purity on $A$ gives $E=O(e^{-T_n})$. Identity
\eqref{eq:app-chord-identity} gives exactly
\begin{equation}
  R_z=R_b^{1/2}R_\ell R_r^2\le E^{7/2}.
  \label{eq:app-z-odds}
\end{equation}
Thus the source is equally pure on $U$. Concentration on the same bounded
feature point implies \eqref{eq:app-source-persistent-means}.
\end{proof}

\begin{lemma}[Persistence through a depth-$2$ child]
\label{lem:app-hidden-persistence}
Let $H_n$ be either a source or a fixed finite-fan-in depth-$2$ node. Assume
all tastes in $H_n$ and its source predecessors are $O(T_n)$ and
\[
  1-p_{H_n}^A(x)=O(e^{-T_n}).
\]
Then
\begin{equation*}
  \min_{S\in\{B_\ell,B_r,U\}}
       [1-p_{H_n}^S(x)]\longrightarrow0.
\end{equation*}
No bound on the incoming edge biases is required.
\end{lemma}

\begin{proof}
The source case is Lemma~\ref{lem:app-source-persistence}. Suppose $H_n$
has $m$ source predecessors. If the conclusion fails, pass to a subsequence
on which the $x$-error on all three off-menus is at least a fixed
$\delta>0$.

On $A$, choose a predecessor $Q_n$ with maximum routing logit. Its routing
weight is at least $1/m$, so it is itself exponentially pure at $x$. After
finitely many subsequence extractions, its identity is fixed, and
Lemma~\ref{lem:app-source-persistence} makes it exponentially pure at $x$
on all four menus.

For each off-menu $S$, choose a predecessor $R_S$ for which
\begin{equation*}
  \alpha_{R_SH}^S[1-p_{R_S}^S(x)]\ge\delta/m,
\end{equation*}
and again fix its identity along a subsequence. Both factors are at least
$\delta/m$. The source $R_S$ has wrong mass bounded away from zero on $A$:
this is immediate for a deletion menu, and for $U$ follows from the
contrapositive of \eqref{eq:app-z-odds}. Since $H_n$ has exponentially small
$A$-error, $R_S$ has routing weight $O(e^{-T_n})$ on $A$. Comparing the
hidden softmax weights of $R_S$ and $Q$ on $A$ and on $S$ cancels both edge
biases and yields
\begin{equation}
  w_{H_n}^\top\left[(\mu_Q^A-\mu_Q^S)
       +(\mu_{R_S}^S-\mu_{R_S}^A)\right]
  \ge T_n-O(1).
  \label{eq:app-hidden-bias-cancel}
\end{equation}

Pass to a further subsequence with $w_{H_n}/T_n\to t=(p,q)$ and all
required source means convergent. The first mean difference in
\eqref{eq:app-hidden-bias-cancel} is $O(e^{-T_n})$. Therefore, writing
$\nu_R^M=\lim_n\mu_R^M$, for all three off-menus
\begin{equation}
  t^\top(\nu_{R_S}^S-\nu_{R_S}^A)>0.
  \label{eq:app-three-switches}
\end{equation}
For each $a\in\{\ell,r\}$, set $S=B_a$ and let $\lambda_{a,n}$ be
the probability assigned to $a$ by its blocker $R_S$ on $A$. Pass to a
further subsequence on which both $\lambda_{a,n}$ converge to
$\lambda_a^*\in[0,1]$. The conditioning identities
\eqref{eq:app-cond-left}--\eqref{eq:app-cond-right} give
\[
  \mu_{R_S}^S-\mu_{R_S}^A
  =\lambda_{a,n}(\mu_{R_S}^S-a),
  \qquad
  \nu_{R_S}^S-\nu_{R_S}^A
  =\lambda_a^*(\nu_{R_S}^S-a).
\]
Positivity in \eqref{eq:app-three-switches} implies $\lambda_a^*>0$
and $t^\top(\nu_{R_S}^S-a)>0$. Since
$\nu_{R_S}^S\in\operatorname{conv}(A\setminus\{a\})$, some
$y\in A\setminus\{a\}$ satisfies $t^\top(y-a)>0$.
For $a=\ell$, these possible inner products are $2p+2q$, $4p+2q$,
and $6p$, so
\begin{equation*}
  \max\{2p+2q,\ 4p+2q,\ 6p\}>0;
\end{equation*}
for $a=r$, they are $-4p+2q$, $-2p+2q$, and $-6p$, so
\begin{equation*}
  \max\{-4p+2q,\ -2p+2q,\ -6p\}>0.
\end{equation*}
If $p\ge0$, the second inequality forces $q>p$; if $p\le0$, the first
forces $q>-p$. Hence $q>|p|$. But then $z$ is the strict $t$-minimum of
$A\cup\{z\}$ because
\begin{align*}
  t^\top(b-z)&=6q-p>0,&t^\top(x-z)&=6q+p>0,\\
  t^\top(\ell-z)&=4q-3p>0,&t^\top(r-z)&=4q+3p>0.
\end{align*}
\Eqref{eq:app-cond-up} now gives
$t^\top(\nu_{R_U}^U-\nu_{R_U}^A)\le0$, contradicting
\eqref{eq:app-three-switches}.
\end{proof}

\paragraph{Twelve rotated copies.}
Let $R_\varphi$ denote planar rotation through the angle $\varphi$. For
$j=0,\ldots,11$, set
\begin{equation*}
  \varphi_j=j\pi/6,\qquad \lambda_j=2^j,\qquad
  \tau_j=(10^6j,0),\qquad
  y_j=\tau_j+\lambda_jR_{\varphi_j}y
\end{equation*}
for every base point $y\in\{b,x,\ell,r,z\}$, and copy the four menus and
targets. This produces $60$ distinct points and $48$ menus. Each copy
inherits the base gadget's properties. A similarity
$x\mapsto\tau+\lambda Rx$ is an affine isomorphism, so it preserves
convex hulls and hence local exposure; and the model class is
equivariant under it: replacing every taste $w_v$ by
$\lambda^{-1}Rw_v$ reproduces every score up to a per-node constant
(equal to $\lambda^{-1}(Rw_v)^\top\tau$, the same for all alternatives
scored by $v$) that cancels in each softmax, so all choice
probabilities are unchanged while taste norms rescale by the factor
$\lambda^{-1}$. Lemma~\ref{lem:app-hidden-persistence} therefore
applies verbatim to every copy, with its norm scale rescaled by
$\lambda_j$---a constant per copy, harmless to the superlogarithmic
conclusion. Copy
$j$ lies within $5\lambda_j\le10240$ of $\tau_j$, whereas distinct centers
are at least $10^6$ apart, so all catalog points are distinct.

Condition (T) also holds. Within one copy, only the two triangles have equal
cardinality, and they are not translates. Across copies, equal-cardinality
menus have diameter $\sqrt{20}\lambda_j$, $6\lambda_j$, or
$\sqrt{37}\lambda_j$ according to their type. Distinct scales give
distinct diameters.

For $\chi=(u,w)$ in the cone
\begin{equation*}
  \mathcal C_0:=\{(u,w):u>0,\ w<-u\},
\end{equation*}
$b$ is the unique $\chi$-minimum on $A$, since
\begin{align*}
  \chi^\top(x-b)&=2u,&\chi^\top(\ell-b)&=-2(u+w),\\
  \chi^\top(r-b)&=4u-2w,&\chi^\top(z-b)&=u-6w
\end{align*}
are positive. On the unit circle this cone is the interval
$(-\pi/2,-\pi/4)$. Copy $j$ is bad on
\[
  I_j=(j\pi/6-\pi/2,\ j\pi/6-\pi/4).
\]
Each interval has width $\pi/4$ and consecutive starts are separated by
$\pi/6$, so their union covers every nonzero direction, with strict
overlap.

\begin{proposition}[A superlogarithmic depth-$3$ instance]
\label{prop:depth3-superlog}
There are $60$ distinct points in $\R^2$, an admissible family
$\mathcal F_\circ$ of $48$ menus, and a deterministic table $c_\circ$ such
that the following holds for every fixed finite topology $G$ of depth at
most $3$, with arbitrary menu-independent edge biases. If
\[
  \mathcal C_G(\epsilon)
  :=\inf\{\|\theta\|_{\rm taste}:
      \err_{\mathcal F_\circ}(G,\theta;c_\circ)\le\epsilon\},
\]
then
\[
  \lim_{\epsilon\downarrow0}
  \frac{\mathcal C_G(\epsilon)}{\log(1/\epsilon)}=+\infty.
\]
\end{proposition}

\begin{proof}[Proof of Proposition~\ref{prop:depth3-superlog}]
Suppose the conclusion fails for a fixed finite depth-$3$ topology $G$.
Then for some $C<\infty$ there are $\epsilon_n\downarrow0$ and parameters
$\theta_n$ satisfying, with $T_n=\log(1/\epsilon_n)$,
\begin{equation*}
  \err_{\mathcal F_\circ}(G,\theta_n;c_\circ)\le\epsilon_n,
  \qquad
  \max_v\|w_v^{(n)}\|_2\le CT_n.
\end{equation*}
This follows from the definition of the infimum by allowing an additive
$1$ in the budget and absorbing it into $CT_n$. Edge biases remain entirely
unrestricted. Pass to a subsequence such that
$w_{v^\star}^{(n)}/T_n\to\chi\in\R^2$. If $\chi\ne0$, choose a copy whose bad
cone contains $\chi$; if $\chi=0$, choose any copy. Suppress the copy index.

If the sink is a source, its exponentially pure choice of $x$ on $A$
persists to all three off-menus, contradicting their target $b$. Thus the
sink is internal. Let its finite fan-in be $W$. On $A$, choose an exact
maximum-logit predecessor $H_n$. Its routing weight is at least $1/W$, so
the final error bound implies
\begin{equation*}
  1-p_{H_n}^A(x)\le W\epsilon_n.
\end{equation*}
Pass to a subsequence on which $H_n=H$. Because the graph has depth at most
three, $H$ is a source or a depth-$2$ node.
Lemma~\ref{lem:app-hidden-persistence} and finitely many further
subsequence extractions give
a fixed $S\in\{B_\ell,B_r,U\}$ with $p_H^S(x)\to1$.

On $S$, choose an exact maximum-logit predecessor $K_n$ and fix it as $K$
along a subsequence. Its routing weight is at least $1/W$, so
$p_K^S(b)\to1$ and $\mu_K^S\to b$. Since $x$ is wrong on $S$, the sink
weight of $H$ is at most $2\epsilon_n$ for all large $n$. Hence
\begin{equation*}
  g_{Kv^\star}^S-g_{Hv^\star}^S\ge T_n-O(1).
\end{equation*}
On $A$, exact maximality of $H$ gives
$g_{Kv^\star}^A-g_{Hv^\star}^A\le0$. Subtracting cancels the two root edge
biases exactly:
\begin{equation}
  T_n-O(1)
  \le(w_{v^\star}^{(n)})^\top
  [ (\mu_K^S-\mu_K^A)-(\mu_H^S-\mu_H^A) ].
  \label{eq:app-root-cancellation}
\end{equation}
Pass to a further subsequence with
$\mu_K^A\to m\in\operatorname{conv}(A)$. The other limits are
$\mu_H^A,\mu_H^S\to x$ and $\mu_K^S\to b$. Dividing
\eqref{eq:app-root-cancellation} by $T_n$ yields
\begin{equation*}
  1\le \chi^\top(b-m).
\end{equation*}
For $\chi=0$ this is impossible. Otherwise the chosen copy makes $b$ the
unique $\chi$-minimum on $A$, so $\chi^\top(b-m)\le0$ for every
$m\in\operatorname{conv}(A)$, again a contradiction. Therefore no bounded
subsequence of
$\mathcal C_G(\epsilon)/\log(1/\epsilon)$ exists, which proves the stated
limit.
\end{proof}

\subsection{Proof of Theorem~\ref{thm:universality4}}
\label{app:proof-universality4}

\subsubsection{Supporting Lemmas}

Fix the observed family $\mathcal F$, with $|\mathcal F|=m$.  If every
observed menu is a singleton, the table is forced and zero tastes suffice.
Assume henceforth that at least one non-singleton menu is present.  Choose a
vector $v$ that strictly orders every observed menu, write
\begin{equation*}
  y_S:=\argmin_{i\in S}v^\top x_i,
  \qquad
  \Gamma:=
  \min_{\substack{S\in\mathcal F\\i\in S\setminus\{y_S\}}}
       v^\top(x_i-x_{y_S})>0,
\end{equation*}
and, for each $A\in\mathcal F$ and $i\in A\setminus\{y_A\}$, choose an
exposing direction $a_{A,i}$ with no ties on any observed menu, and its
trace $q_{A,i}$ as in \eqref{eq:app-hard-trace}.

\begin{lemma}[A common bounded reference source]
\label{lem:app-bounded-sentinel}
There is a finite taste $r_0\in\R^d$ such that
\begin{equation}
  0\le v^\top\bigl(\mu_S(r_0)-x_{y_S}\bigr)<\Gamma/8
  \qquad(S\in\mathcal F),
  \label{eq:app-sentinel-near-worst}
\end{equation}
and, simultaneously for every $A\in\mathcal F$,
$i\in A\setminus\{y_A\}$, and $S\in\mathcal F\setminus\{A\}$,
\begin{equation}
  x_i-\mu_A(r_0)\ne
  x_{q_{A,i}(S)}-\mu_S(r_0).
  \label{eq:app-sentinel-noncollision}
\end{equation}
\end{lemma}

\begin{proof}
For every fixed radius $R_0>0$, the means at tastes in the ball
$B(-Kv,R_0)$ converge uniformly, as $K\to\infty$, to the unique
$v$-worst points.  Hence, for large finite $K$, this ball is a nonempty
open set on which \eqref{eq:app-sentinel-near-worst} holds.  For a fixed
triple $(A,i,S)$, failure of
\eqref{eq:app-sentinel-noncollision} is the vector-analytic level set
\[
  \mu_A(r)-\mu_S(r)=x_i-x_{q_{A,i}(S)}.
\]
If this level set contained an open set,
Lemma~\ref{lem:app-translation-characterization} would make $A$ and $S$
exact translates.  Here $A$ is non-singleton; this contradicts (T) when
$S$ is non-singleton, and unequal cardinalities when $S$ is a singleton.
Each bad set therefore has empty interior.  Finitely many such analytic
zero sets cannot cover the open ball, so a common $r_0$ exists.
\end{proof}

Let $R$ denote the reference source with taste $r_0$. For an active row
$A$---that is, $c(A)\ne y_A$---let $Q_A$ be the source with taste
$Ta_{A,c(A)}$.

\begin{lemma}[Pair gate]
\label{lem:app-pair-gate}
For every active $A$ and every $S\ne A$, a binary node $G_{A,S}$ receiving
$Q_A$ and $R$ can be parameterized so that, uniformly as $T\to\infty$, it
routes to $Q_A$ on $A$ and to $R$ on $S$, with error $O(e^{-cT})$.
\end{lemma}

\begin{proof}
Put $i=c(A)$ and define
\[
  d_A:=x_i-\mu_A(r_0),
  \qquad
  d_S:=x_{q_{A,i}(S)}-\mu_S(r_0),
  \qquad h_{A,S}:=d_A-d_S\ne0.
\]
Give $G_{A,S}$ taste $Th_{A,S}$, bias
\begin{equation*}
  T\beta_{A,S},
  \qquad
  \beta_{A,S}:=-\tfrac12h_{A,S}^\top(d_A+d_S),
\end{equation*}
on its $Q_A$ edge, and bias zero on its $R$ edge.  In the limit, the
$Q_A$-versus-$R$ logit differences on $A$ and $S$ are respectively
\[
  T(h_{A,S}^\top d_A+\beta_{A,S})
  =\frac{T}{2}\|d_A-d_S\|_2^2,
  \qquad
  T(h_{A,S}^\top d_S+\beta_{A,S})
  =-\frac{T}{2}\|d_A-d_S\|_2^2.
\]
The source $Q_A$ converges exponentially to its hard trace, so the exact
finite-$T$ logits differ from these expressions by $O(Te^{-c'T})$.
The claimed routing and output estimates follow.
\end{proof}

For each active row $A$, create a \emph{soft-min node} $D_A$ of taste
$-Tv$. It receives $Q_A$ directly and every $G_{A,S}$, $S\ne A$, with
zero bias on all incoming edges. On $A$, all these inputs are
$O(e^{-cT})$-close to $Q_A$, so their mixture at $D_A$ is too. On an
off-row $S$, the designated input $G_{A,S}$ is $O(e^{-cT})$-close to
$R$. The node favors inputs with smaller projected scores
$v^\top\mu$; the following estimate bounds its output score even when
incoming scores tie.

\begin{lemma}[Gibbs soft minimum]
\label{lem:app-soft-minimum}
For $T>0$ and $u_1,\ldots,u_J\in\R$, put
\[
  \lambda_j:=\frac{e^{-Tu_j}}{\sum_ke^{-Tu_k}}.
\]
Then
\begin{equation*}
  0\le\sum_j\lambda_ju_j-\min_j u_j\le\frac{\log J}{T}.
\end{equation*}
\end{lemma}

\begin{proof}
Subtract $u_*:=\min_j u_j$.  If $H(\lambda)$ denotes Shannon entropy,
the Gibbs identity gives
\[
  T\sum_j\lambda_j(u_j-u_*)
  =H(\lambda)-\log\sum_je^{-T(u_j-u_*)}
  \le H(\lambda)\le\log J.
\]
\end{proof}

Apply the lemma to $u_j=v^\top\mu_j^S$ for the incoming branches of
$D_A$. Since there are $m$ such branches and one is exponentially close
to the reference source $R$, for every $S\ne A$,
\begin{equation}
  v^\top\mu_{D_A}^S
  \le v^\top\mu_S(r_0)+\frac{\log m}{T}+O(e^{-cT}).
  \label{eq:app-sanitizer-safe}
\end{equation}

\subsubsection{Proof of Theorem~\ref{thm:universality4}}
The topology has one foreground source $Q_A$ for every
$A\in\mathcal F$, the reference source $R$, a separate pure-baseline
source $P$, one pair gate $G_{A,S}$ for every ordered pair $A\ne S$,
one soft-min node $D_A$ for every $A$, and one sink. Its edges are
\[
  Q_A,R\longrightarrow G_{A,S},
  \qquad
  Q_A,\{G_{A,S}:S\ne A\}\longrightarrow D_A,
  \qquad
  D_A,P,R\longrightarrow v^\star.
\]
The edge $R\to v^\star$ merely keeps the graph weakly connected when
$m=1$ and will be suppressed.  This topology is fixed from the abstract
family before the features or table are seen.  It has depth at most $4$
and exactly
\begin{equation*}
  m+2+m(m-1)+m+1=m^2+m+3
\end{equation*}
vertices.

After seeing the table, parameterize every active $Q_A$, its pair gates,
and $D_A$ as above. Set the tastes of inactive $Q_A$, $G_{A,S}$, and
$D_A$ nodes, and the incoming biases at these nodes, to zero. Give $P$
taste $-Tv$ and the sink taste $Tv$. At the sink, give $P$ edge bias
$T\Gamma/2$, every active $D_A$ edge bias zero, and every inactive
$D_A$ edge and the $R$ edge bias $-TB$, where
$B>\max_{i,j}v^\top(x_i-x_j)+\Gamma$.

On an active row $A$, all inputs to $D_A$ are exponentially close to
$Q_A$, and therefore, for all large $T$,
\[
  v^\top(\mu_{D_A}^A-\mu_P^A)-\Gamma/2\ge\Gamma/4.
\]
Thus its own soft-min node beats the baseline by sink logit margin at
least $T\Gamma/4$. If the same node is evaluated on an off-row $S$,
Lemma~\ref{lem:app-bounded-sentinel},
\eqref{eq:app-sanitizer-safe}, and exponential concentration of $P$ give
\begin{align*}
  v^\top(\mu_{D_A}^S-\mu_P^S)-\Gamma/2
  &\le \Gamma/8-\Gamma/2+\frac{\log m}{T}+o(1)\\
  &\le-\Gamma/4
\end{align*}
for all large $T$. Hence, on an active row its own soft-min node is the
unique winning relevant branch; on a baseline row $S$ with $c(S)=y_S$,
every active soft-min node loses to $P$. The suppressed branches lose
uniformly.

All foreground and pure-baseline source errors and all designated
pair-gate errors are $O(e^{-cT})$. On its own row, a soft-min node mixes
only inputs close to the same foreground law. On an off-row, the
$O(1/T)$ term in \eqref{eq:app-sanitizer-safe} contributes only $O(1)$
to a sink comparison with leading margin of order $T$. Finiteness of
the graph and Lemma~\ref{lem:app-mixture-error} therefore give constants
$C_1,c_1,C_2>0$ and $T_0\ge1$, independent of $c$, such that for
$T\ge T_0$,
\begin{equation*}
  \err_{\mathcal F}(G_{\mathcal F}^{(4)},\theta_T;c)
  \le C_1e^{-c_1T},
  \qquad
  \|\theta_T\|_{\rm taste}\le C_2T.
\end{equation*}
The pair-gate and sink edge biases are also $O(T)$. Taking
$T=\max\{T_0,c_1^{-1}\log(C_1/\epsilon)\}$ gives error at most
$\epsilon$ and taste norm $O(\log(1/\epsilon))$ as
$\epsilon\downarrow0$. Choose a fixed
$\epsilon_0\in(0,1-1/N)$ small enough that this norm bound holds for
$0<\epsilon\le\epsilon_0$. For
$\epsilon_0<\epsilon<1-1/N$, reuse the parameters for $\epsilon_0$
and enlarge the constant, using
$\log(1/\epsilon)\ge\log(N/(N-1))>0$. For
$1-1/N\le\epsilon<1$, set all tastes and edge biases to zero. Every
menu then has the uniform choice law, whose error is at most $1-1/N$.
Thus the logarithmic norm bound holds for every $0<\epsilon<1$.
Since the fixed family has finitely many tables, one constant works
for all $c$.

For optimality, take $0<\epsilon<1/2$ and suppose $\mathcal F$ contains
a non-singleton menu $S$ on which the sink assigns its target $i$
probability at least $1-\epsilon$.
Lemma~\ref{lem:app-source-mixture} gives an upstream source $s$ with the
same probability bound, and the argument of
Lemma~\ref{lem:app-source-purity} applies verbatim: for any
$j\in S\setminus\{i\}$,
\[
  e^{w_s^\top(x_i-x_j)}
  =\frac{p_s^S(i)}{p_s^S(j)}
  \ge\frac{1-\epsilon}{\epsilon}.
\]
Writing $\operatorname{diam}(X):=\max_{k,\ell}\|x_k-x_\ell\|_2$ yields
\begin{equation*}
  \|\theta\|_{\rm taste}\ge\|w_s\|_2
  \ge\frac{1}{\operatorname{diam}(X)}
       \log\frac{1-\epsilon}{\epsilon}.
\end{equation*}
Thus the logarithmic order is optimal whenever a non-singleton menu is
observed.  If every menu is a singleton, the table is forced and zero
taste norm suffices.

\section{Experimental Details}
\label{app:protocol}

\subsection{Datasets}
\label{app:datasets}

Attribute features are $z$-normalized using train-pool statistics;
the one-hot item indicators in group A below are used directly.

\paragraph{Expedia.}
Source: the ICDM 2013 ``Personalize Expedia Hotel Searches'' data
\citep{expedia2013}. We keep searches with at least one booking,
price at most $\$1{,}000$, booking window at most $365$ days, and
menu size at least $2$; the booked hotel is the observed choice, and
no outside option is modeled. Sorted listings are kept, so
\texttt{position} is an endogenous feature---the same caveat as
Trivago below. Features ($d=7$): \texttt{position},
\texttt{prop\_starrating}, \texttt{prop\_location\_score1},
\texttt{prop\_log\_historical\_price}, \texttt{prop\_brand\_bool},
\texttt{promotion\_flag}, \texttt{price\_usd}; $B_{\max}=37$. The
split forms $75/25$ pools by search time and samples $60{,}000$
train / $20{,}000$ test sets (seed $42$); the validation set is the
last $10\%$ of the train pool by time. A random-forest two-sample
classifier distinguishes train-period from test-period session summaries
with five-fold cross-validated AUC $0.84$; logistic regression gives $0.81$.

\paragraph{Trivago.}
Source: the RecSys 2019 challenge logs \citep{knees2019recsys}, US
platform only. One observation per session: the last click-out
event with recorded impressions, whose impression list (up to $25$ hotels with displayed
prices) is the assortment and whose clicked hotel is the choice; the
clicked item must appear in the list and the menu must have at least
$2$ items. After filtering: $81{,}519$ sessions, split $75/25$ by
time into $61{,}139$ train / $20{,}380$ test sessions (used in full,
no subsampling); the validation set is the last $10\%$ of the train
pool by time. Features ($d=7$) are \texttt{position},
\texttt{log\_price}, \texttt{stars}, \texttt{rating},
\texttt{n\_properties}, \texttt{free\_wifi}, and \texttt{hotel\_flag}.
The log spans six days (2018-11-01 to 11-06). Under the same five-fold
diagnostic protocol, random forest gives AUC $0.56$ and logistic
regression gives $0.52$. Both detect less temporal shift than on
Expedia. \emph{Caveat}:
\texttt{position} reflects the platform's ranking and is endogenous;
all models receive the same features, so the comparison remains
internally valid.

\paragraph{Fixed-pool benchmarks.}
Group A (featureless; items are one-hot indicators of the pool):
\emph{SFwork} ($5{,}029$ observations, $6$ commute modes, menu sizes
$3$--$6$) and \emph{SFshop} ($3{,}157$ observations, $8$
alternatives, menu sizes $4$--$8$), the San Francisco Bay Area 1990
travel-survey pair released with the PCMC paper; cross-sectional
split. Group B (numeric attributes; item identity enters only
through intercepts): \emph{Swissmetro} ($10{,}719$ observations, $3$
alternatives, $d=4$: travel time, cost with the GA rail-pass
discount, headway, seats; panel data, split by respondent) and
\emph{ModeCanada} ($4{,}324$ cases, $4$ alternatives, $d=4$: cost,
in-vehicle time, out-of-vehicle time, frequency; menu sizes
$2$--$4$). All splits are $75/25$ with seed $42$, group-aware where
groups exist, with validation $=10\%$ of train groups.

\subsection{Unified Training Protocol}
\label{app:training}

The benchmark experiments use the following training protocol; the
synthetic decoy fit is specified separately in Appendix~\ref{app:decoy-fit}.
We train with Adam at
batch size $256$ on the NLL of the observed choices, at most $150$
epochs, early stopping on validation NLL with patience $15$ (minimum
$30$ epochs), and the best-validation weights restored at the end.
The depth ladder, fixed-pool rows, and fixed-configuration rows of
Table~\ref{tab:deep} use learning rate $10^{-2}$ and $10$ restarts.
For MMNL, PCMC, and MCCM on the fixed pools, the archived runs also
include $10^{-3}$; $10^{-2}$ has the better mean validation NLL on
all four datasets. The DeepMNL, RUMnet, DeepHalo, and TCNet grids
use $3$ restarts per configuration and select by mean validation NLL
(Appendix~\ref{app:configs}); DeepSets uses a fixed configuration
and $10$ restarts. The encoder--head matrix uses $10$ restarts and
one learning rate per encoder, held fixed across heads.
Each reported entry averages all restarts of one configuration,
without selecting individual restarts. The fixed-pool LC-MNL and
LowRankCDM size comparisons use the reporting rule specified in
Appendix~\ref{app:fixedpool-grid-details}.
Standard deviations use divisor $n$ for the fixed-pool tables and
$n-1$ for the hotel experiments, matching the archived summaries.
Seeds are deterministic. Table~\ref{tab:deep} reuses the MNL,
LC-MNL, DeepSets, and circuit runs from the corresponding cells of
Table~\ref{tab:head}; other experiments retain their recorded seed
namespaces. The code package provides per-run records and the
mapping from each reported table to these records.
Experiments ran on an Apple M4 CPU, with TCNet and DeepHalo
trained on the MPS backend and the fixed-pool sweep on a SLURM
cluster.

\subsection{Structures of the Depth Ladder
(Table~\ref{tab:depth-ladder})}
\label{app:structures}

All ladder graphs have $63$ nodes, arranged in levels of widths
$62{-}1$ (depth $2$), $55{-}7{-}1$ (depth $3$),
$45{-}13{-}4{-}1$ (depth $4$), $36{-}16{-}7{-}3{-}1$ (depth $5$),
and $32{-}16{-}8{-}4{-}2{-}1$ (depth $6$, the complete binary
tree). Trees connect each node to one parent in the next level;
dense graphs connect every node to every node of the next level. At
$d=7$ the trees all have $503$ parameters; the dense graphs pay one
bias per edge, for $503/833/1082/1153/1123$ parameters at depths
$2$--$6$.

\subsection{Model Configurations}
\label{app:configs}

The grids for Table~\ref{tab:deep} are listed below; selection uses
mean validation NLL across the three restarts of each configuration.
Here E denotes Expedia and T denotes Trivago.

\begin{center}\footnotesize
\setlength{\tabcolsep}{4pt}
\begin{tabular}{lll}
\toprule
Model & grid & selected \\
\midrule
DeepMNL (E) & depth $2$, width $\in\{16,64\}$, lr $10^{-2}$ & width $64$ \\
DeepMNL (T) & width $64$, depth $\in\{2,3\}$, lr $10^{-3}$ & depth $2$ \\
RUMnet & $K\in\{5,10\}$, lr $10^{-2}$ & $K=10$ on both \\
DeepHalo & $H=4$, $L=2$, embed $128$, lr $10^{-3}$ & fixed on both \\
TCNet (E) & width $32$, $4$ heads; $(L1,10^{-3})$, $(L2,10^{-2})$ & $L1$ \\
TCNet (T) & $L2$, $4$ heads, width $\in\{32,64\}$, lr $10^{-2}$ & width $32$ \\
\bottomrule
\end{tabular}
\end{center}

Parameter counts at $d=7$: DeepMNL $4{,}737$; RUMnet $4{,}510$;
DeepHalo $201{,}097$; TCNet $20{,}001$ ($L1$) / $37{,}217$ ($L2$).
The encoders of Table~\ref{tab:head} are trained jointly
(end-to-end) with their heads under the protocol above:
$\mathrm{MLP}(16)$ is an item-wise single-hidden-layer encoder
($145$ parameters); $\mathrm{MLP}(64,64)$ is the DeepMNL trunk
applied item-wise ($4{,}737$); DeepSets is a pooled-context set
encoder ($3{,}681$); these counts include each encoder's native
linear readout. TCNet and DeepHalo use the selected configurations
above. For all six heads, the learning rate is $10^{-2}$ for Id,
MLP(16), and DeepSets, and $10^{-3}$ for MLP(64,64) and DeepHalo;
TCNet uses $10^{-3}$ on Expedia and $10^{-2}$ on Trivago.
On the fixed pools, the classics carry item intercepts on
group B (ASCs), our trees carry per-item source biases ($+$ib), and
on group A the one-hot input makes intercepts implicit.
ConsiderMNL is the classical independent-consideration model
\citep{swait1987incorporating,benakiva1995discrete}: item
$i$ enters the consideration set with probability
$\sigma(a_i+b^{\top}x_i)$, the choice is an MNL among the considered
items conditioned on a nonempty set, and at these pool sizes the
$2^{|A|}$ consideration sets are enumerated exactly. PCMC and the
Markov chain model are trained with the same Adam$+$ES recipe (the
literature uses ILSR resp.\ EM on the same likelihood).

\subsection{Parameter Counts on the Fixed Pools}
\label{app:fixedpool-params}

\begin{table}[H]
\centering\scriptsize
\setlength{\tabcolsep}{5pt}
\begin{tabular}{l cccc}
\toprule
model & SFwork & SFshop & Swissmetro & ModeCanada\\
\midrule
MNL                      & 6 & 8 & 4 & 4\\
MNL (with intercepts)    & --- & --- & 7 & 8\\
NestedLogit              & 9 & 11 & 9 & 11\\
LC-MNL(2)                & 14 & 18 & 16 & 18\\
LC-MNL(3)                & 21 & 27 & 24 & 27\\
LC-MNL(8)                & 56 & 72 & 64 & 72\\
LC-MNL(16)               & 112 & 144 & 128 & 144\\
LC-MNL(32)               & 224 & 288 & 256 & 288\\
MMNL                     & 12 & 16 & 11 & 12\\
Exponomial               & 6 & 8 & 7 & 8\\
ConsiderMNL              & 18 & 24 & 14 & 16\\
CDM-id                   & 42 & 72 & 16 & 24\\
LowRankCDM(1)            & 18 & 24 & 15 & 16\\
LowRankCDM(2)            & 30 & 40 & 23 & 24\\
LowRankCDM(3)            & 42 & 56 & 31 & 32\\
LowRankCDM(4)            & 54 & 72 & 39 & 40\\
PCMC                     & 36 & 64 & 9 & 16\\
MCCM                     & 42 & 72 & 12 & 20\\
$\mathrm{Tree}_{15}$     & 104 & 134 & 98 & 106\\
$\mathrm{Tree}_{63}$     & 440 & 566 & 410 & 442\\
\bottomrule
\end{tabular}
\caption{Parameter counts behind Table~\ref{tab:fixedpool}. On the
featured pair the tree counts include the per-item source biases.}
\label{tab:fixedpool-params}
\end{table}

\subsection{Full Size Grids (Table~\ref{tab:fixedpool})}
\label{app:fixedpool-grid-details}

For LC-MNL and LowRankCDM, Table~\ref{tab:fixedpool} reports the
lowest mean test NLL over the size grid below. This gives each
baseline its strongest reported size; it is a comparison across
tested sizes, not a validation-selected estimate. Training
checkpoints are still selected by validation NLL. Selecting size
by mean validation NLL instead changes only two entries:
LowRankCDM on SFwork selects rank $3$, giving $0.8213$ ($0.0010$),
and LC-MNL on Swissmetro selects $K=8$, giving $0.7522$ ($0.0091$).
The code package exports both summaries. Neither changes which
model has the lowest mean NLL on any of the four datasets.

\begin{table}[H]
\centering\small
\setlength{\tabcolsep}{6pt}
\begin{tabular}{l cc cc}
\toprule
& \multicolumn{2}{c}{featureless} & \multicolumn{2}{c}{featured}\\
\cmidrule(lr){2-3}\cmidrule(lr){4-5}
model & SFwork & SFshop & Swissmetro & ModeCanada\\
\midrule
LC-MNL(2)      & \ms{0.8275}{.0014} & \ms{1.5872}{.0049} & \ms{0.7838}{.0029} & \ms{0.6098}{.0023}\\
LC-MNL(3)      & \ms{0.8265}{.0008} & \ms{1.5844}{.0058} & \ms{0.7628}{.0082} & \ms{0.6125}{.0075}\\
LC-MNL(8)      & \ms{0.8263}{.0002} & \ms{1.5735}{.0089} & \ms{0.7522}{.0091} & \ms{0.6021}{.0057}\\
LC-MNL(16)     & \ms{0.8264}{.0004} & \ms{1.5670}{.0005} & \ms{0.7518}{.0083} & \ms{0.5977}{.0025}\\
LC-MNL(32)     & \ms{0.8262}{.0005} & \ms{1.5669}{.0010} & \ms{0.7536}{.0073} & \ms{0.5929}{.0034}\\
\midrule
LowRankCDM(1)  & \ms{0.8226}{.0028} & \ms{1.5747}{.0080} & \ms{0.7635}{.0014} & \ms{0.5843}{.0027}\\
LowRankCDM(2)  & \ms{0.8217}{.0021} & \ms{1.5693}{.0093} & \ms{0.7573}{.0023} & \ms{0.5851}{.0032}\\
LowRankCDM(3)  & \ms{0.8213}{.0010} & \ms{1.5681}{.0075} & \ms{0.7581}{.0024} & \ms{0.5921}{.0081}\\
LowRankCDM(4)  & \ms{0.8211}{.0015} & \ms{1.5648}{.0017} & \ms{0.7573}{.0009} & \ms{0.5900}{.0069}\\
\bottomrule
\end{tabular}
\caption{Full size grids behind the LC-MNL and LowRankCDM rows of
Table~\ref{tab:fixedpool} (test NLL, mean {\scriptsize(sd)} over $10$
restarts).}
\label{tab:fixedpool-grid}
\end{table}

\subsection{The Full Encoder-by-Head Matrix}
\label{app:head-matrix}

Figure~\ref{fig:pipeline} illustrates the encoder--circuit
composition used in the output-head experiments.
Tables~\ref{tab:matrix-expedia} and~\ref{tab:matrix-trivago} report
every encoder--head combination behind Table~\ref{tab:head} and the
gated-readout comparison of \S\ref{sec:head}. Heads:
\emph{linear} (the encoder's native readout), $\mathrm{Mix}(8)$ (a
constant mixture of eight MNLs), $\mathrm{MoE}(32)$ (a gated
mixture over $32$ MNLs), $\mathrm{HME}(63)$ (a hierarchical
mixture of experts on the $63$-node binary tree), and our
$\mathrm{Tree}_{15}$/$\mathrm{Tree}_{63}$.

\begin{figure}[htbp]
\centering
\includegraphics[width=370.92784119pt,height=101.98229980pt]{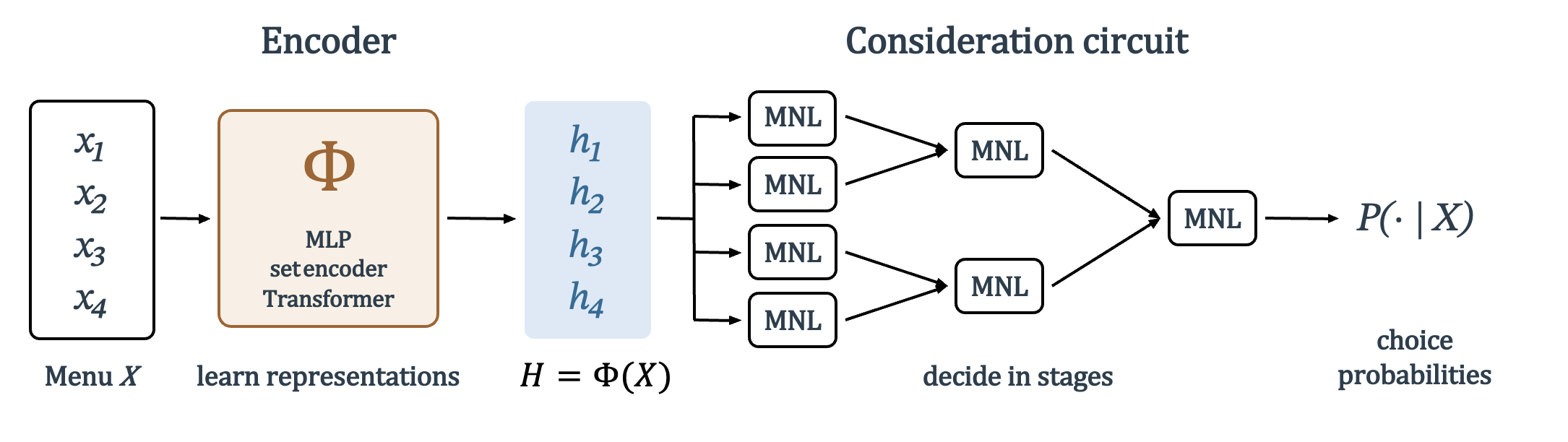}
\caption{The encoder $\Phi$ maps the menu to transformed item
representations $h_i$, which are passed to a consideration-circuit
output head.}
\label{fig:pipeline}
\end{figure}

\begin{table}[H]
\centering\scriptsize
\setlength{\tabcolsep}{2.5pt}
\resizebox{\textwidth}{!}{\begin{tabular}{l cccccc}
\toprule
encoder & linear & Mix(8) & MoE(32) & HME(63)
& $\mathrm{Tree}_{15}$ & $\mathrm{Tree}_{63}$\\
\midrule
Id& \ms{2.6574}{.0008}& \ms{2.6153}{.0006}& \ms{2.6991}{.0506}& \ms{2.7017}{.0693}& \ms{2.6009}{.0026}& \ms{2.5983}{.0018}\\
MLP(16)& \ms{2.6798}{.0398}& \ms{2.6786}{.0215}& \ms{2.7462}{.0997}& \ms{2.7450}{.1234}& \ms{2.6515}{.0349}& \ms{2.6469}{.0152}\\
MLP(64,64)& \ms{2.7006}{.0401}& \ms{2.7224}{.0748}& \ms{2.7887}{.0786}& \ms{2.8022}{.0585}& \ms{2.6607}{.0138}& \ms{2.6621}{.0181}\\
DeepSets& \ms{2.9113}{.2282}& \ms{2.8807}{.2945}& \ms{2.7157}{.0731}& \ms{2.8083}{.1816}& \ms{2.8132}{.1640}& \ms{2.8957}{.1787}\\
DeepHalo& \ms{2.8272}{.1744}& \ms{2.6854}{.0367}& \ms{2.7594}{.0953}& \ms{2.8182}{.1155}& \ms{2.6490}{.0337}& \ms{2.6351}{.0117}\\
TCNet& \ms{2.7377}{.0686}& \ms{2.7063}{.0398}& \ms{2.6809}{.0413}& \ms{2.6852}{.0342}& \ms{2.6834}{.0380}& \ms{2.6611}{.0322}\\
\bottomrule
\end{tabular}}
\caption{Encoder $\times$ head on Expedia (temporal split): test
NLL, mean {\scriptsize(sd)} over $10$ restarts.}
\label{tab:matrix-expedia}
\end{table}

\begin{table}[H]
\centering\scriptsize
\setlength{\tabcolsep}{2.5pt}
\resizebox{\textwidth}{!}{\begin{tabular}{l cccccc}
\toprule
encoder & linear & Mix(8) & MoE(32) & HME(63)
& $\mathrm{Tree}_{15}$ & $\mathrm{Tree}_{63}$\\
\midrule
Id& \ms{2.6400}{.0006}& \ms{2.4631}{.0006}& \ms{2.4322}{.0013}& \ms{2.4348}{.0015}& \ms{2.4417}{.0015}& \ms{2.4331}{.0023}\\
MLP(16)& \ms{2.4809}{.0007}& \ms{2.4491}{.0018}& \ms{2.4197}{.0024}& \ms{2.4202}{.0027}& \ms{2.4229}{.0032}& \ms{2.4185}{.0027}\\
MLP(64,64)& \ms{2.4751}{.0010}& \ms{2.4469}{.0009}& \ms{2.4156}{.0010}& \ms{2.4156}{.0016}& \ms{2.4194}{.0019}& \ms{2.4148}{.0016}\\
DeepSets& \ms{2.4179}{.0032}& \ms{2.4224}{.0097}& \ms{2.4077}{.0024}& \ms{2.4085}{.0021}& \ms{2.4119}{.0115}& \ms{2.4114}{.0118}\\
DeepHalo& \ms{2.4151}{.0027}& \ms{2.4273}{.0032}& \ms{2.4129}{.0040}& \ms{2.4164}{.0051}& \ms{2.4141}{.0024}& \ms{2.4133}{.0024}\\
TCNet& \ms{2.4093}{.0031}& \ms{2.4067}{.0021}& \ms{2.4053}{.0017}& \ms{2.4045}{.0030}& \ms{2.4035}{.0049}& \ms{2.4013}{.0027}\\
\bottomrule
\end{tabular}}
\caption{Encoder $\times$ head on Trivago: test NLL, mean
{\scriptsize(sd)} over $10$ restarts.}
\label{tab:matrix-trivago}
\end{table}

\subsection{Scale Ablation: Is Robustness Just Small Parameter
Count?}
\label{app:scale}

To test whether parameter count explains the Expedia results in
\S\ref{sec:vs-deep}, we increase circuit size and reduce the size of
TCNet and DeepHalo under the same training protocol. We evaluate
complete binary trees $\mathrm{Tree}_{127}$ through
$\mathrm{Tree}_{4095}$ and fully connected layered circuits
$\mathrm{FC}[n_1,\dots,n_k]$, where $n_1,\dots,n_k$ are the level
widths. Every node connects to every node in the next level; the
edge-bias matrix between levels $\ell$ and $\ell+1$ has
$n_{\ell+1}\times n_\ell$ entries, for $\ell=1,\dots,k-1$.
We also reduce TCNet and DeepHalo to $1.5$k--$5.8$k parameters.
All configurations use the protocol of \S\ref{app:training} with
$10$ restarts.

\paragraph{Level-synchronous evaluation.}
For large circuits, we batch the recursive computation of
\S\ref{sec:model} by topological level. Sources are at level $0$;
each interior node is one level above its highest-level predecessor.
All nodes in a level are evaluated in a batched gather--score--softmax
step. In-neighborhoods are padded to a common size, and padded entries
receive score $-\infty$ in the softmax. This computes the same nodewise
recursion while making the number of batched steps depend on circuit
depth rather than node count.

Table~\ref{tab:scale} reports test NLL. Across $311$--$44{,}167$
parameters, circuit test NLL ranges from $2.5973$ to $2.5996$ and
remains below the MNL baseline. The reduced TCNet and DeepHalo
models still have higher test NLL than MNL. Thus parameter count
alone does not explain the performance difference: the comparison
supports the circuit's staged choice structure as an effective
inductive bias on this temporal split.

\begin{table}[H]
\centering\small
\setlength{\tabcolsep}{6pt}
\begin{tabular}{l r c}
\toprule
model & params & test NLL\\
\midrule
MNL & 8 & \ms{2.6574}{.0008}\\
\midrule
$\mathrm{Tree}_{127}$  & 1{,}015  & \ms{2.5985}{.0030}\\
$\mathrm{Tree}_{255}$  & 2{,}039  & \ms{2.5991}{.0029}\\
$\mathrm{Tree}_{511}$  & 4{,}087  & \ms{2.5980}{.0018}\\
$\mathrm{Tree}_{1023}$ & 8{,}183  & \ms{2.5981}{.0015}\\
$\mathrm{Tree}_{2047}$ & 16{,}375 & \ms{2.5976}{.0023}\\
$\mathrm{Tree}_{4095}$ & 32{,}759 & \ms{2.5973}{.0016}\\
\midrule
$\mathrm{FC}[16,8,1]$        & 311      & \ms{2.5996}{.0015}\\
$\mathrm{FC}[32,16,1]$       & 871      & \ms{2.5992}{.0018}\\
$\mathrm{FC}[64,32,1]$       & 2{,}759  & \ms{2.5986}{.0022}\\
$\mathrm{FC}[64,64,64,64,1]$ & 14{,}151 & \ms{2.5977}{.0011}\\
$\mathrm{FC}[256,128,64,1]$  & 44{,}167 & \ms{2.5991}{.0026}\\
\midrule
TCNet (w$8$)     & 1{,}545 & \ms{2.6852}{.0506}\\
TCNet (w$16$)    & 5{,}393 & \ms{2.7215}{.0545}\\
DeepHalo (e$16$) & 1{,}603 & \ms{2.8416}{.1421}\\
DeepHalo (e$32$) & 5{,}763 & \ms{2.8310}{.2132}\\
\bottomrule
\end{tabular}
\caption{The scale ablation on the Expedia temporal axis: test NLL,
mean {\scriptsize(sd)} over $10$ restarts; MNL row from
Table~\ref{tab:deep}. Shrunk configurations: TCNet width $8/16$, one
layer, two heads; DeepHalo embed $16/32$, two heads, one round.}
\label{tab:scale}
\end{table}

\section{Decoy Fit Details}
\label{app:decoy-fit}

\paragraph{Instance.}
Three items described by quality and price, with target choice
probabilities exhibiting the decoy effect
(Table~\ref{tab:decoy-instance}): item $3$ is dominated by item $2$
(lower quality at a higher price) but not by item $1$, and its entry
raises item $2$'s share from $50\%$ to $60\%$. Features are
$z$-normalized using the mean and population standard deviation
across all three items, with the same transformation used for both menus:
\[
  \bar x=(8.5,\,385/3),\qquad
  \sigma=(1/\sqrt{6},\,\sqrt{3650}/3).
\]

\begin{table}[H]
\centering\small
\begin{tabular}{lcccc}
\toprule
item & quality & price & $P(\cdot\mid\{1,2\})$ & $P(\cdot\mid\{1,2,3\})$ \\
\midrule
$1$          & $8.0$ & $\$100$ & $50\%$ & $35\%$ \\
$2$          & $9.0$ & $\$140$ & $50\%$ & $60\%$ \\
$3$ (decoy)  & $8.5$ & $\$145$ & ---    & $\phantom{0}5\%$ \\
\bottomrule
\end{tabular}
\caption{The decoy instance: features and target choice
probabilities.}
\label{tab:decoy-instance}
\end{table}

\paragraph{Training.}
The smallest consideration circuit (two sources and a sink; $d=2$) is fit
with both edge biases fixed to zero, so six taste parameters are
trained, by minimizing the population NLL---the
cross-entropy between the target and model distributions, averaged
over the two menus with equal weight---using Adam (learning rate
$0.05$, $2{,}500$ steps) from $80$ random initializations with
scales in $\{0.1,0.5,1,2\}$, followed by a polish of the best run
($30{,}000$ steps, learning rate $0.01$). Using the fitted parameters,
the mean KL divergence from the target distributions is below
$10^{-14}$, and the maximum absolute choice-probability error is below
$5\times10^{-8}$.

\paragraph{Fitted parameters.}
In $z$-scored feature coordinates (quality, price), the stored
double-precision parameters are
\begin{align*}
  w_A &= (-1.0101429662092325,\, -0.3345334746021139),\\
  w_B &= (+4.3142002721689385,\, -1.671384212169921),\\
  w_0 &= (+2.4650095689934712,\, -3.0847402667907686).
\end{align*}
Table~\ref{tab:decoy-fit-params} reports the per-source choice
distributions, sink attention, and resulting shares, rounded for
display. Figure~\ref{fig:decoy} shows selected values at lower precision.

\begin{table}[H]
\centering\small
\begin{tabular}{lcccc}
\toprule
 & source $A$ & source $B$ & model & target \\
\midrule
\multicolumn{5}{@{}l}{\emph{menu} $\{1,2\}$}\\
$P(1)$ & $0.958$ & $0.001$ & $50.0\%$ & $50\%$ \\
$P(2)$ & $0.042$ & $0.999$ & $50.0\%$ & $50\%$ \\
sink attention & $0.521$ & $0.479$ & & \\
\midrule
\multicolumn{5}{@{}l}{\emph{menu} $\{1,2,3\}$}\\
$P(1)$ & $0.847$ & $0.001$ & $35.0\%$ & $35\%$ \\
$P(2)$ & $0.037$ & $0.996$ & $60.0\%$ & $60\%$ \\
$P(3)$ & $0.116$ & $0.003$ & $\phantom{0}5.0\%$ & $\phantom{0}5\%$ \\
sink attention & $\bm{0.413}$ & $\bm{0.587}$ & & \\
\bottomrule
\end{tabular}
\caption{The fitted model on the two menus: per-source choice
distributions, sink attention, and model vs.\ target shares.}
\label{tab:decoy-fit-params}
\end{table}

\end{APPENDICES}
\end{document}